\documentclass[11pt]{article}
\usepackage[margin=1in]{geometry}
\usepackage{amsmath,amssymb,amsthm}
\usepackage[T1]{fontenc}
\usepackage[utf8]{inputenc}
\usepackage{tgtermes}
\usepackage{newtxmath}
\usepackage{booktabs}
\usepackage{hyperref}
\usepackage{enumitem}
\newtheorem{theorem}{Theorem}[section]

\newtheorem{corollary}[theorem]{Corollary}
\theoremstyle{definition}

\theoremstyle{remark}
\newtheorem{remark}[theorem]{Remark}

\newcommand{\Omin}{\mathcal O_{\min}}
\newcommand{\Omax}{\mathcal O_{\max}}

\title{An Exponential Deterministic--Randomized Gap\\ in ERM-Oracle Complexity for Thresholds on an Unknown Order}
\author{Xuan Li\\[2pt]
{\normalsize University of New South Wales, Sydney, Australia}\\
{\normalsize \texttt{winny.li@unsw.edu.au}\quad \href{https://orcid.org/0009-0002-0213-6991}{ORCID: 0009-0002-0213-6991}}}
\date{}

\begin{document}
\maketitle

\begin{abstract}
Attias, Hanneke and Ramaswami (NeurIPS 2025) asked whether randomization provably reduces the oracle calls needed for online learning when the class is accessible only through an oracle. We study the instance they singled out: transductive online learning of thresholds on an unknown total order of T instances, with a consistency-type ERM oracle that returns a full concept consistent with a queried labeled set (or reports non-realizability). Our main result is a separation for a fixed natural oracle. When the oracle is the minimal-prefix rule (or the maximal-prefix rule), every deterministic learner makes M mistakes and Q calls with $M+Q\ge T-\varepsilon$ on some instance ($\varepsilon\in\{0,1\}$, according to whether the empty prefix is a concept), and the constant is exact; hence $O(\log T)$ mistakes cost $T-\varepsilon-O(\log T)$ calls, whereas that paper's randomized learner achieves $O(\log T)$ expected calls and mistakes under the same rule. The randomized order is optimal: on an explicit hard distribution under the minimal-prefix rule, every learner has expected mistakes at least $((T+1-\varepsilon)\,128^{-\mathbb{E}[Q]}-1)/2$, so $\Omega(\log T)$ expected calls are necessary for polylogarithmic mistakes. The separation is governed by the oracle's selection rule, not by the class alone: for a legal feasible-median ERM rule a deterministic learner achieves $O(\log T)$ calls and mistakes, while a global-median rule again forces linear total cost. The same linear bound holds when the oracle's answers are chosen adversarially and then frozen into a memoryless oracle. We add partial tradeoff results for fixed query budgets (the middle regime is open) and an interface contrast: with only a weak consistency oracle, returning a realizability bit, both deterministic and randomized learners need $\Theta(T)$ calls.
\end{abstract}

\section{Introduction}
\label{sec:intro}

In oracle-based online learning the learner does not see the concept class directly. It interacts with the class only through an oracle, and two resources are then distinguished: the number of prediction mistakes, which measures statistical difficulty, and the number of oracle calls, which measures how much of the class the learner must actually inspect. A class of small Littlestone dimension may be learnable with few mistakes and yet require many oracle calls, and the calls themselves can be the expensive part. Which oracle interface is offered, and how the oracle chooses among many legal answers, then becomes part of the problem specification.

Attias, Hanneke and Ramaswami \cite{AHR25} (henceforth AHR25) study online and transductive online learning with two such interfaces: a consistency-type ERM oracle, which given a labeled finite set $S$ returns some concept of the class consistent with $S$ (as a full label vector) or reports that $S$ is not realizable, and a weak consistency oracle, which returns only the realizability bit. For the family of threshold classes on an unknown total order of the instances they prove (their Theorem~4.5) that a deterministic learner can achieve $O(\log T)$ mistakes with $O(T)$ ERM calls, and that a randomized learner can achieve $O(\log T)$ mistakes with $O(\log T)$ \emph{expected} ERM calls. They write that the latter ``represents an exponential improvement over deterministic algorithms'' (AHR25, \S1.1); no lower bound for deterministic learners on this family is proved there, and their concluding section asks:
\begin{quote}
``Are randomized algorithms provably more powerful than deterministic ones for online learning with oracles?''
\end{quote}
(The abstract of AHR25 states that ``$O(\log T)$ ERM queries suffice'' for thresholds on an unknown ordering; the algorithm in question is randomized and the bound is in expectation.)

\paragraph{A separation for a fixed natural oracle.}
Our main result answers this question for the instance the authors singled out, and it does so for a single oracle that is announced before the learner. Let the ERM oracle be the \emph{minimal-prefix rule}: on a realizable sample it returns the smallest consistent prefix of the unknown order. This is arguably the most natural consistency-type ERM oracle for thresholds, it is memoryless, and it is public. Theorem~\ref{thm:Aprime} states that for every deterministic learner there are an order and a target on which the number of mistakes plus the number of calls is at least $T-\varepsilon$, where $\varepsilon\in\{0,1\}$ records whether the empty prefix counts as a concept; the constant is exact, and the same holds for the maximal-prefix rule and for any sample-dependent choice between the two. Consequently a deterministic learner with $O(\log T)$ mistakes must make $T-\varepsilon-O(\log T)$ calls, whereas the randomized learner of AHR25, which works for every legal oracle, makes $O(\log T)$ expected calls with $O(\log T)$ mistakes under the very same rule. In the Littlestone dimension $d=\Theta(\log T)$ of the class this is a gap of $O(d+1)$ versus $2^{\Omega(d)}$.

\paragraph{Why a full concept does not help a deterministic learner.}
The oracle returns a complete label vector, so one call can reveal the relative order of many instances at once, and the sorting-based deterministic algorithm of AHR25 does exploit this. The lower bound nevertheless charges each call for at most one point. The reason is a \emph{one-pin common-value} property of the extremal rules: the adversary maintains a block $L\prec F\prec R$ of instances with $F$ freely permutable, and for every query it can move at most one point of $F$ to an endpoint of the block so that the minimal-prefix answer is the same full vector on every remaining order. Every prediction on a still-free point is then forced to be wrong, and since only calls remove free points without a mistake, mistakes plus calls add up to the number of free points. Because the rule is fixed in advance, the final instance is simply one of the retained orders and the declared oracle needs no record of past answers.

\paragraph{Randomized calls are optimal in order.}
Theorem~\ref{thm:C} shows that the randomized $O(\log T)$ of AHR25 cannot be improved in order: on an explicit distribution over orders and targets, under the minimal-prefix rule, every randomized learner satisfies $\mathbb E[M]\ge((T+1-\varepsilon)\,128^{-\mathbb E[Q]}-1)/2$, so polylogarithmic expected mistakes require $\Omega(\log T)$ expected calls. Here a single call may legitimately reveal a linear number of labels, so a one-point charging argument is unavailable; instead a logarithmic potential (``paid predictions plus remaining hidden points'') decreases by at most $7\ln 2$ per call in expectation. The natural uniform-random-order distribution provably cannot yield such a bound.

\paragraph{The oracle's selection rule is a complexity axis.}
The linear deterministic bound is not a property of the threshold class alone. Theorem~\ref{thm:rule} exhibits a legal, memoryless ``feasible-median'' ERM rule under which a deterministic learner achieves $\lceil\log_2 T\rceil$ calls and $\lceil\log_2 T\rceil$ mistakes (for $T\ge2$), while a ``global-median'' rule again forces linear total cost. The exponential advantage of randomization is therefore an advantage against \emph{unfavourable} selection (extremal or adversarial), and randomization substitutes for it: the randomized learner is indifferent to how the oracle selects, whereas a deterministic learner is at the mercy of the rule. We regard this as the conceptual content of the paper: for oracle-based online learning, the selection behaviour of the oracle is a complexity parameter alongside the interface and the class.

\paragraph{Contributions.}
Throughout, $T$ is the number of instances, $M$ the number of mistakes, $Q$ the number of oracle calls, and $\varepsilon=0$ under convention~E (the empty prefix on the instances is a concept) and $\varepsilon=1$ under convention~N (it is not); see Section~\ref{sec:model}.
\begin{enumerate}[leftmargin=2em]
\item \textbf{Fixed-natural-oracle separation (Theorem~\ref{thm:Aprime}).} Under the pre-declared minimal-prefix rule (or the maximal-prefix rule, or any sample-dependent choice between the two), every deterministic learner has $M+Q\ge T-\varepsilon$ on some instance, and $\inf_{\mathcal A}\sup(M+Q)=T-\varepsilon$ exactly. Only the order and the target are chosen adversarially; the oracle precedes the learner, matching the literal reading of ``an ERM oracle for $C$''.
\item \textbf{Matching randomized lower bound (Theorem~\ref{thm:C}).} On an explicit finite distribution under the minimal-prefix rule, every randomized learner satisfies $\mathbb E[M]\ge((T+1-\varepsilon)\,128^{-\mathbb E[Q]}-1)/2$. With AHR25, the randomized ERM query complexity of $O(\log T)$ mistakes is $\Theta(\log T)$.
\item \textbf{Selection-rule dependence (Theorem~\ref{thm:rule}).} The linear bound does not extend to every legal pre-declared rule: the feasible-median rule admits deterministic $\lceil\log_2T\rceil$ calls and mistakes for $T\ge2$, whereas the global-median rule forces $M+Q\ge T-\lfloor(\varepsilon+T)/2\rfloor$. Every linear deterministic ERM lower bound in this paper is a statement about the extremal rules of Theorem~\ref{thm:Aprime} and about worst-case legal oracles; it is not uniform over all legal pre-declared selection rules.
\end{enumerate}
Three further results support and refine the picture. Theorem~\ref{thm:A} shows that the bound $M+Q\ge T-\varepsilon$ also holds when the oracle's answers are chosen adversarially after the learner and then frozen into a memoryless oracle, which is the operational form of the lower bounds in AHR25. Theorem~\ref{thm:B} gives deterministic lower and upper bounds for fixed query budgets, with exact values in the small- and large-budget regimes and finite counterexamples showing that the naive exact formula fails in between. Theorem~\ref{thm:D} contrasts interfaces: with only a weak consistency oracle, both deterministic and randomized learners need $\Theta(T)$ calls for $O(\log T)$ mistakes, so the exponential advantage of randomization is tied to the oracle returning an evaluable concept.

\paragraph{Techniques and what is portable.}
Two lower-bound templates carry the paper. The deterministic template (Section~\ref{sec:det}) is a common-value argument: a partial commitment on the order, a lemma showing that one endpoint pin per query stabilizes the oracle's full answer, an additive charging of free points, and a fixation step that turns the adaptive construction into one fixed instance. Only the one-pin lemma uses the threshold structure and the extremal rule; the commitment, charging and fixation steps apply verbatim to any pre-declared rule with the one-pin common-value property (Appendix~\ref{sec:proofAprime}, \S6), and, with an answer cache, to adversarial oracles (Appendix~\ref{sec:proofA}). The randomized template (Section~\ref{sec:rand}) is a hard-distribution potential argument: a hidden-center interval-tree distribution, an auxiliary oracle that reveals strictly more than the real one, a posterior-invariance lemma, and a bound of $O(1)$ on the expected decrease of a logarithmic potential per call. The interval-tree geometry is threshold-specific; the device of paying for revealed labels with already-made predictions is not. Neither template uses the general lower-bound theorem of AHR25 (their Theorem~4.4); Remark~\ref{rem:44} explains once why it does not apply here.

\paragraph{Organization.}
Section~\ref{sec:model} fixes the model, the default convention and the quantifier forms. Section~\ref{sec:results} states the three headline results. Sections~\ref{sec:det} and~\ref{sec:rand} sketch the deterministic and randomized lower bounds, and Section~\ref{sec:det} also states the adversarial-then-frozen theorem. Section~\ref{sec:rule} discusses selection-rule dependence. Section~\ref{sec:cons} states the fixed-budget tradeoff bounds and the weak-consistency interface contrast. Section~\ref{sec:open} lists open problems. Complete proofs are in Appendices~\ref{sec:proofAprime}--\ref{sec:proofD}.

\subsection{Related work}
\label{sec:related}

\emph{Mistake bounds and the transductive protocol.} The mistake-bound model and the Littlestone dimension are due to Littlestone \cite{Lit88}; the transductive (``off-line'') variant, in which the sequence of instances is announced in advance, was introduced by Ben-David, Kushilevitz and Mansour \cite{BKM97}, and Hanneke, Moran and Shafer \cite{HMS23} show that the optimal transductive mistake bound is always $\Theta(1)$, $\Theta(\log T)$ or $T$. In those models the learner knows the class; the number of mistakes is the only resource. For thresholds on a known order of $T$ points the optimal mistake bound is $\lfloor\log_2(T+1-\varepsilon)\rfloor$, and this remains the information-theoretic benchmark here; the question studied in this paper is what it costs, in oracle calls, to reach it when the order is unknown and the class is accessible only through an oracle.

\emph{Oracle-efficient online learning.} Assos et al.\ \cite{AAD23} learn online with an ERM oracle that returns a concept minimizing empirical error on a submitted sample, and Kozachinskiy and Steifer \cite{KS24} with a consistent oracle that returns some concept agreeing with the examples seen so far; both obtain mistake bounds depending only on the Littlestone dimension. In those works the learner submits a sample, the oracle returns a concept, and the resource optimized is the number of mistakes; the number of calls is not the object of study. AHR25 made the number of calls itself the resource, treated the weak consistency interface alongside the ERM interface, and proved the upper bounds recalled above; earlier oracle-efficient results are surveyed there. Daskalakis and Golowich \cite{DG24} study the weak consistency oracle (a realizability bit) in the PAC setting and show that efficient PAC learning is possible with it; their setting is distributional rather than online, and their oracle answers about a fixed sample rather than an adaptively grown one. Syrgkanis, Krishnamurthy and Schapire \cite{SKS16} obtain oracle-efficient adversarial contextual learning from an optimization oracle; the interaction between an oracle's tie-breaking and the learner appears there too, in a different model.

\emph{Regret--oracle tradeoffs and learning orders.} In the agnostic, non-transductive setting against adaptive adversaries, Attias, Hanneke and Ramaswami \cite{AHR26} prove query-budget regret lower bounds (their Theorem~5.1); for Littlestone dimension $d=1$ and $1\le Q=o(\sqrt T)$ these imply $\Omega(T/Q)$ (their general form $\Omega(dT/Q)$ needs $d\le Q+1$ and $Q=o(\sqrt{dT})$, and $Q=0$ is covered only by the original formula). The setting, the class and the cost measure differ from ours. Alon, Moran and Moran \cite{AMM26} learn an unknown linear order from counterexamples: the learner proposes a complete order and receives either confirmation or a wrongly ordered pair (with up to $k$ untruthful answers), and the query complexity is $\Theta(n\log n+nk)$. Their goal is the order itself; ours is to predict threshold labels through an ERM or realizability oracle for the class, and the cost is mistakes plus calls. In the classical query models of Angluin \cite{Ang88} the oracle answers questions about the target (membership, equivalence); the oracles here answer questions about the class (consistency), and the target is revealed only through the online labels. Our results say that a randomized learner need not learn the order at all ($O(\log T)$ calls), whereas a deterministic one must pay a linear price against worst-case legal oracles and under the extremal rules of Theorem~\ref{thm:Aprime} (Theorems~\ref{thm:Aprime} and~\ref{thm:A}; not under every legal rule, Theorem~\ref{thm:rule}). We do not claim any equivalence between these models; they are cited to place the selection behaviour of an oracle, as a design parameter of the feedback, in a broader landscape.

\section{Model and conventions}
\label{sec:model}

Instances $x_1,\dots,x_T$ are distinct and announced in advance; the adversary fixes a total order $\preceq$ on a finite domain $X$ and a target concept $c_z(x)=\mathbf 1[x\preceq z]$ with $z\in X$; the threshold class is $C_\preceq=\{c_z:z\in X\}$ and the family is $\mathcal F=\{C_\preceq\}$. Labels arrive one per round; the learner predicts $\hat y_t$ before seeing $y_t=c_z(x_t)$. Restricted to the instances, every concept is the indicator of a $\preceq$-prefix.

\paragraph{Default convention and its correction.} By default we use convention~E: the domain is $X=\{s\}\cup\{x_1,\dots,x_T\}$, where the sentinel $s$ precedes all instances in every order considered, so the empty prefix on the instances is a concept; there are $T+1$ concepts and the Littlestone dimension is $\lfloor\log_2(T+1)\rfloor$. Convention~N, in which the domain is exactly the instance set $X=\{x_1,\dots,x_T\}$ (there are $T$ concepts; AHR25 write $\lfloor\log_2 T\rfloor$), differs from E by the bookkeeping constant $\varepsilon$: the lower bounds of Theorems~\ref{thm:Aprime}, \ref{thm:C} and \ref{thm:A} are stated with $\varepsilon\in\{0,1\}$ and hold under both conventions, while Theorems~\ref{thm:rule}, \ref{thm:B} and \ref{thm:D} state the two conventions explicitly. All lower bounds are proved on these finite domains, and all upper bounds hold on them. The numbers of concepts differ by one between the two conventions, while their Littlestone dimensions differ by at most one (for $T=5$ both equal $2$); we write $d=\lfloor\log_2(T+1-\varepsilon)\rfloor$ for the Littlestone dimension, so that $T-\varepsilon\ge 2^d-1$. The parallel treatment of convention~N, including the sentinel bookkeeping, is confined to the appendices.

\paragraph{Oracles.} A \emph{consistency-type ERM oracle} takes any finite $S\subseteq X\times\{0,1\}$ and returns some $c\in C_\preceq$ consistent with $S$, observable as a full label vector, or $\bot$ if none exists (AHR25, Def.~2.1, second variant). A \emph{weak consistency oracle} returns only whether such a $c$ exists (AHR25, Def.~2.2). Every call costs one, regardless of $|S|$; calls may be made at any time. The learner knows $\mathcal F$, $X$, $T$ and the sequence, but not $\preceq$ or $z$. A \emph{selection rule} is a function $\mathcal R(S,\preceq)$ that is a legal consistency-type ERM oracle for every order; the \emph{minimal-prefix rule} $\Omin(S,\preceq)$ returns the smallest consistent prefix and the \emph{maximal-prefix rule} $\Omax$ the largest. Degenerate queries: $S=\varnothing$ is realizable; $S$ assigning both labels to one point is not; under E the sentinel $s$ has label $1$ in every concept.

\paragraph{Who chooses the oracle's answer.} The lower bounds in this paper come in three quantifier forms, which we fix here once and refer to throughout.
\begin{center}
\fbox{\begin{minipage}{0.94\linewidth}\small
\textbf{Fixed rule} (Theorems~\ref{thm:Aprime}, \ref{thm:rule}): the selection rule $\mathcal R$ is announced first. For every deterministic learner there exist an order and a target: $\forall\mathcal A\ \exists(\preceq,z)$, with oracle $S\mapsto\mathcal R(S,\preceq)$.

\smallskip
\textbf{Adversarial then frozen} (Theorem~\ref{thm:A}): for every deterministic learner there exist an order, a target and a memoryless legal oracle $O^\star$, a function of the sample alone, on which the learner is replayed: $\forall\mathcal A\ \exists(\preceq,z,O^\star)$. The oracle is produced after the learner. (AHR25 note that for deterministic learners an oblivious adversary is as powerful as an adaptive one; the freezing step is where this is made precise.)

\smallskip
\textbf{Randomized learners} (Theorem~\ref{thm:C}): the instance (order, target, sequence, selection rule) is fixed before the learner's private randomness (oblivious adversary), and costs are worst-instance expectations over that randomness.
\end{minipage}}
\end{center}
The fixed-rule form is the literal reading of ``an ERM oracle for $C$'', in which the oracle precedes the learner; the adversarial-then-frozen form is the operational standard of the lower bounds in AHR25. The randomized learner of AHR25 works for every legal oracle, so the fixed-rule form is the one that yields a separation with the oracle fixed first.

\section{Main results}
\label{sec:results}

\begin{theorem}[Deterministic lower bound; pre-declared extremal rule]
\label{thm:Aprime}
Let $T\ge1$ and fix either convention. Let $\mathcal R$ be $\Omin$, or $\Omax$ (the largest consistent prefix), or any rule $\mathcal R_\eta(S,\preceq)$ that applies $\Omin$ or $\Omax$ according to a function $\eta(S)$ of the sample alone; $\mathcal R$ is announced before the learner and the instance.
\begin{enumerate}[label=(\alph*),leftmargin=2em]
\item For every deterministic learner $\mathcal A$ and every instance sequence there exist $\preceq^\star$ and $z^\star$ such that, with oracle $S\mapsto\mathcal R(S,\preceq^\star)$, $M(\mathcal A)+Q(\mathcal A)\ge T-\varepsilon$.
\item Consequently $\inf_{\mathcal A}\sup_{\preceq,z}\big(M+Q\big)=T-\varepsilon$ (the upper bound needs no calls: under E any fixed prediction rule, under N always predicting $1$).
\end{enumerate}
\end{theorem}

Consequently, if a deterministic learner guarantees at most $m(T)$ mistakes on every instance under one of these rules, its worst-case number of calls is at least $T-\varepsilon-m(T)$; an $O(\log T)$-mistake guarantee costs $T-\varepsilon-O(\log T)$ calls. The randomized learner of AHR25 (Thm.~4.5(4)) achieves $O(\log T)$ expected calls with $O(\log T)$ mistakes against every legal oracle, hence under the same fixed rule. In terms of $d=\lfloor\log_2(T+1-\varepsilon)\rfloor$, the deterministic worst-case total cost is at least $T-\varepsilon\ge2^d-1$ while the randomized expected total cost is $O(d+1)$. The statement concerns only the instance singled out in AHR25 (thresholds on an unknown order, consistency-type ERM oracle with full label vectors, transductive protocol, the finite domains of Section~\ref{sec:model}); it does not claim a uniform improvement for arbitrary classes or for the non-transductive setting.

\begin{theorem}[Randomized lower bound]
\label{thm:C}
For every $T\ge1$ and either convention there exist a finite distribution over (total order, target), a fixed instance sequence, and the memoryless minimal-prefix oracle $\Omin(\cdot,\preceq)$, such that every randomized learner satisfies
\[
\mathbb E[M]\ \ge\ \frac{(T+1-\varepsilon)\,128^{-\mathbb E[Q]}-1}{2},
\]
the expectations being over the distribution and the learner's randomness. In particular, if $\mathbb E[Q]\le q$ (or $Q\le q$ almost surely) then $\mathbb E[M]\ge\max\{0,((T+1-\varepsilon)128^{-q}-1)/2\}$, and there is a fixed instance on which this holds for any learner with a uniform expected-query guarantee $q$.
\end{theorem}

\begin{corollary}
\label{cor:C}
If a randomized learner has worst-instance expected mistakes $m_T$ and expected calls $q_T$, then $q_T\ge\big(\ln(T+1-\varepsilon)-\ln(2m_T+1)\big)/(7\ln2)$; for $m_T=O((\log T)^d)$ this is $\Omega(\log T)$. With AHR25 (Thm.~4.5(4)) the randomized ERM query complexity of $O(\log T)$ mistakes is $\Theta(\log T)$, and with Theorems~\ref{thm:Aprime} and~\ref{thm:A} and the deterministic $O(T)$-call algorithm of AHR25 (Thm.~4.5(3)) the deterministic one is $\Theta(T)$ under the extremal rules of Theorem~\ref{thm:Aprime} and against worst-case legal oracles.
\end{corollary}

\begin{theorem}[Selection-rule dependence]
\label{thm:rule}
The statement of Theorem~\ref{thm:Aprime} does not extend to every pre-declared memoryless rule. Let $\mathcal R_{\rm feas}$ return the prefix whose cutoff is the midpoint $\lfloor(\ell+u)/2\rfloor$ of the interval $[\ell,u]$ of feasible cutoffs (and $\bot$ if the interval is empty). $\mathcal R_{\rm feas}$ is a legal memoryless consistency-type ERM oracle, and for $T\ge2$ there is a deterministic learner with $Q\le\lceil\log_2T\rceil$ and $M\le\lceil\log_2 T\rceil$ on every instance; for $T=1$ no query is needed and the learner makes at most one mistake under E and none under N (under E, $T=1$ has two legal targets, so zero mistakes cannot be guaranteed by any learner). For the rule $\mathcal R_{\rm global}$ that returns $\bot$ on unrealizable samples and otherwise projects a fixed global median cutoff $m=\lfloor(k_0+T)/2\rfloor$ onto the feasible cutoff interval $[\ell,u]$, i.e.\ returns cutoff $\min\{u,\max\{\ell,m\}\}$ ($k_0=\varepsilon$), one query and $\lceil T/2\rceil$ mistakes suffice, while every deterministic learner has $M+Q\ge T-m$ on some instance; so no deterministic learner achieves $O(\log T)$ mistakes and $O(\log T)$ calls under $\mathcal R_{\rm global}$. These are upper and lower bounds, not the exact minimax values of these rules.
\end{theorem}

Table~\ref{tab:interfaces} summarizes the query complexity of $O(\log T)$ mistakes at the two interfaces, with the oracle quantifier made explicit.

\begin{table}[h]
\centering
\small
\begin{tabular}{p{0.34\linewidth}p{0.28\linewidth}p{0.3\linewidth}}
\toprule
Oracle interface & Deterministic & Randomized (expected) \\
\midrule
Consistency-type ERM (returns a full concept) & $\Theta(T)$: Thm.~\ref{thm:Aprime}, \ref{thm:A}; AHR25 Thm.~4.5(3) & $\Theta(\log T)$: AHR25 Thm.~4.5(4); Thm.~\ref{thm:C} \\
Weak consistency (returns a realizability bit) & $\Theta(T)$: Thm.~\ref{thm:D}; AHR25 Thm.~4.3 & $\Theta(T)$: AHR25 Thm.~4.5(2), 4.3; Lemma~E.4 \\
\bottomrule
\end{tabular}
\caption{Query complexity of achieving $O(\log T)$ mistakes for thresholds on an unknown order. The deterministic ERM entry refers to the pre-declared extremal rules of Theorem~\ref{thm:Aprime} and to worst-case (adversarial) legal oracles; it is not uniform over all legal pre-declared rules: Theorem~\ref{thm:rule} gives a legal ERM rule under which the entry drops to $O(\log T)$.}
\label{tab:interfaces}
\end{table}

\section{The deterministic separation}
\label{sec:det}

We sketch the proof of Theorem~\ref{thm:Aprime} for $\Omin$; the complete argument, with the reduction for $\Omax$ and $\mathcal R_\eta$, is Appendix~\ref{sec:proofAprime}.

\paragraph{Anchored states.} Fix an anchor $\sigma$ that is constrained to be the order-minimum and hence always has label $1$: under E, $\sigma=s$; under N, $\sigma=x_1$. Let $V=X\setminus\{\sigma\}$ and $n=|V|=T-\varepsilon$. The adversary maintains an ordered list $L$, an unordered set $F$ and an ordered list $R$ partitioning $V$; the \emph{retained orders} are $\sigma\cdot L\cdot\pi(F)\cdot R$ for all permutations $\pi$ of $F$, and the \emph{retained targets} cut anywhere through $F$: every point of $L$ is committed to label $1$, every point of $R$ to label $0$, and the target may include any initial segment of the free permutation. Initially $F=V$. The only state changes are \emph{endpoint pins}: a left pin moves $p\in F$ to the end of $L$, a right pin moves it to the front of $R$. Pins shrink the retained set (Appendix~\ref{sec:proofAprime}, (A.5)).

\paragraph{The one-pin common-value lemma (Lemma~A.2).} For every state and every query $S$, at most one endpoint pin makes $\Omin(S,\preceq)$ the \emph{same full vector} on all retained orders. Write $A$ for the positive and $B$ for the negative sample points (after rejecting contradictory samples and $(\sigma,0)$, whose answer is $\bot$). By the prefix criterion (Lemma~A.1), $S$ is realizable in an order iff no negative precedes a positive, and then the minimal prefix ends at $\max_\preceq A$. If some $b\in B$ precedes some $a\in A$ in every retained order, answer $\bot$ without a pin. If $A$ and $B$ both meet $F$, left-pin a free negative $b$: every retained order now has $b$ before every free positive, so the answer is $\bot$. Otherwise the sample is realizable in every retained order and only the selected prefix must be stabilized: if $A$ meets $R$, the last positive of $R$ is $\max A$ in every retained order and its prefix $\{\sigma\}\cup L\cup F\cup R_{\le p}$ is fixed (no pin); if $A$ meets $F$ but not $R$, then $B\subseteq R$, and right-pinning a free positive $p$ makes $p=\max A$ with prefix $\{\sigma\}\cup L\cup F_{\rm old}$ (one pin); if $A\subseteq L$, the prefix through the last positive of $L$ is fixed (no pin). The six cases are exhaustive. The lemma is a statement about a fixed natural rule on genuine total-order thresholds; it does not say that the returned vector carries one bit of information.

\paragraph{Predictions and charging (Lemmas~A.3--A.4).} When the learner predicts $\hat y$ on an instance $x$: if $x\in\{\sigma\}\cup L$ reveal $1$, if $x\in R$ reveal $0$, and if $x\in F$ reveal $1-\hat y$ and pin $x$ to the corresponding side. The transcript invariant is that every recorded answer equals $\Omin(S,\preceq)$ on every currently retained order, that all revealed labels agree with every retained target, and that no free point has been predicted. The potential $\Phi=|F|$ starts at $n$, decreases by at most one per call and decreases at a prediction only through a mistake; when all $T$ predictions are made, $\Phi=0$, so $M+Q\ge n=T-\varepsilon$.

\paragraph{Fixation without a cache (Lemma~A.5).} Simulate the learner until all predictions are made or until the $n$-th call (this covers learners that would call infinitely often). Order the residual $F$ by arrival index and take $\preceq^\star=\sigma\cdot L\cdot\pi_0(F)\cdot R$ and $z^\star$ the last point of $L$ (or $\sigma$ if $L=\varnothing$); this is a retained pair. Run the learner on $(\preceq^\star,z^\star)$ with the declared oracle $S\mapsto\Omin(S,\preceq^\star)$. By the invariant every recorded answer equals $\Omin(S,\preceq^\star)$ and every recorded label equals $c_{z^\star}$, so by determinism the interaction replays exactly, and the bound follows. The oracle is the originally declared function: legal on every sample, memoryless, and consistent on repeated queries by definition. The upper bound in (b) needs no calls: under E any fixed predictions, under N always predicting $1$. For $\Omax$ an anchored reversal-complement duality (Lemma~A.6) maps $\Omax$ to $\Omin$ preserving mistakes and calls; for $\mathcal R_\eta$, $\eta(S)$ is fixed independently of the order, so the minimum or maximum procedure is applied at each query. More generally the argument applies to any fixed legal rule with the \emph{one-pin common-value property}: on every anchored state and every sample, some refinement by zero or one endpoint pin makes the rule's full answer constant on all retained orders. This is a sufficient condition for the linear bound, not a characterization (Section~\ref{sec:rule}).

\subsection{Robustness: adversarial then frozen oracles}
\label{sec:frozen}

The bound of Theorem~\ref{thm:Aprime} is not an artefact of the particular rule: it persists when the adversary chooses the answers.

\begin{theorem}[Deterministic lower bound; adversarial then frozen oracle]
\label{thm:A}
Let $T\ge1$ and fix either convention. For every deterministic learner $\mathcal A$ there exist a total order $\preceq^\star$ on $X$, a point $z^\star\in X$, and a function $O^\star:2^{X\times\{0,1\}}\to C_{\preceq^\star}\cup\{\bot\}$ which is a legal consistency-type ERM oracle on every sample and depends on the sample only, such that on this fixed instance
\[
M(\mathcal A)+Q(\mathcal A)\ \ge\ T-\varepsilon .
\]
The constant cannot be improved uniformly over $T\ge1$.
\end{theorem}

\begin{corollary}
\label{cor:A}
If a deterministic learner guarantees at most $m(T)$ mistakes on all instances and legal memoryless ERM oracles, its worst-case number of calls is at least $T-\varepsilon-m(T)$. In particular an $O(\log T)$-mistake guarantee requires $T-\varepsilon-O(\log T)$ calls, whereas the randomized learner of AHR25 (Thm.~4.5(4)) achieves $O(\log T)$ expected calls with $O(\log T)$ mistakes. In terms of $d=\lfloor\log_2(T+1-\varepsilon)\rfloor$: the deterministic worst-case total cost is at least $T-\varepsilon\ge 2^d-1$ while the randomized expected total cost is $O(d+1)$.
\end{corollary}

The proof (Appendix~\ref{sec:proofA}) is the same one-point charging argument with an adaptive answering rule in place of the fixed one: on a realizable query with no mixed-label free block it returns the prefix of the atom list $(L,[F],R)$ through the rightmost positive atom. Its delicate part is the freezing step. Because this adaptive rule may legally return two different vectors to the same sample at different states (Appendix~\ref{sec:proofA} exhibits such a sample), the adversary records the first answer to each sample in a cache, and the frozen oracle $O^\star$ returns cached answers on recorded samples and the minimal consistent prefix of the final order otherwise. The fixed-rule proof of Theorem~\ref{thm:Aprime} needs no cache. Theorem~\ref{thm:A} has the quantifier form $\forall\mathcal A\ \exists(\preceq,z,O)$ of the box in Section~\ref{sec:model}: the oracle is produced after the learner and then frozen. It does not by itself yield a statement in which one oracle is fixed before all learners; that form is Theorem~\ref{thm:Aprime}.

\begin{remark}[Why Theorem~4.4 of AHR25 is not invoked]
\label{rem:44}
AHR25 prove a general lower bound (their Theorem~4.4) for families of all classes of a given Littlestone dimension $d_{\rm LD}$ by an equivalence-class construction in which ``each query provides information about at most one point''. Neither Theorem~\ref{thm:Aprime} nor Theorem~\ref{thm:A} invokes it. Its hypothesis $T\ge2^{d_{\rm LD}+1}$ fails for the threshold family, whose dimension is $d=\lfloor\log_2(T+1-\varepsilon)\rfloor$; its construction uses thresholds on equivalence classes rather than on distinct totally ordered points; and its one-point property is not literally true for full-vector answers, which may encode many comparisons among already committed points. Our arguments charge at most one \emph{free} point per query, which is a statement about the residual block, not about the information content of the answer. The verbatim statement and the comparison are given once, in Appendix~\ref{sec:proofA}; the theorem is not used anywhere in this paper.
\end{remark}

\section{The randomized lower bound}
\label{sec:rand}

We sketch the proof of Theorem~\ref{thm:C}; details are in Appendix~\ref{sec:proofC}. Logarithms are natural and $n=T-\varepsilon$ is the number of \emph{hard} points; under E the sentinel $a=s$ is the minimum of the order, and under N the remaining instance $a$ (which arrives last) is the minimum of the order and always has label $1$.

\paragraph{The hard distribution.} Identify the hard points with $x_1,\dots,x_n$ in arrival order and build the balanced binary interval tree on $[n]$ (an interval of size $m\ge2$ splits into its first $\lfloor m/2\rfloor$ and remaining $\lceil m/2\rceil$ indices). Draw a hidden center $J$ uniformly from $[n]$ and independent fair labels $Y_1,\dots,Y_n$. The \emph{priority order} $\pi_J$ lists, along the root-to-$J$ path, each sibling interval not containing $J$ (in increasing index order) and finally $J$; thus every sibling met earlier has outer priority to everything in the current interval around $J$. The total order is $a\prec(\text{points with }Y_i=1\text{ in }\pi_J\text{ order})\prec(\text{points with }Y_i=0\text{ in reverse }\pi_J\text{ order})$, and the target ends at the last $Y_i=1$ point (or at $a$). Then $c_z(x_i)=Y_i$: the labels are fair bits, and the geometry of the order is what the oracle can reveal. The oracle is the memoryless minimal-prefix rule; all random choices precede the interaction. This is a genuine total-order threshold class, not a class on equivalence classes.

\paragraph{A stronger auxiliary oracle.} A minimal-prefix call is simulated by one call to a prefix-max oracle that returns the prefix ending at $\max_\preceq$ of the positive sample points (Lemma~C.1). We replace it by an even stronger \emph{revelation oracle} that maintains an interval $I\ni J$ (initially $[n]$) and, on a query with positive set $A$, repeatedly reveals which child of $I$ contains $J$ together with all labels in the other child, stopping as soon as $A\setminus I$ contains a target-negative point, or $A\cap I=\varnothing$, or $I$ is a singleton (which it then reveals); all of this costs \emph{one} call. At every stopping point the requested prefix is determined by the revealed information (Lemma~C.2), so any learner for the real oracle is simulated with the same number of calls and the same predictions. A single call may thus reveal a linear number of labels; this is permitted and is exactly why a one-point charging argument is unavailable.

\paragraph{Posterior invariance and the potential.} Conditional on the entire auxiliary history, $J$ is uniform in the current $I$, the labels of unarrived points in $I$ are independent fair bits, and they are independent of $J$ (Lemma~C.3); this replaces, rather than assumes, exchangeability after $\bot$ answers, and it holds for randomized learners after conditioning on the seed. Call a prediction \emph{unrevealed} if its point is still in $I$ when it is made; such a prediction faces an independent fair bit and costs $1/2$ in expectation, so $\mathbb E[M]\ge\frac12\mathbb E[H_{\rm final}]$ where $H$ counts unrevealed predictions made so far. With $r$ the number of not-yet-predicted points in $I$, the potential is
\[
\Phi=\ln(H+r+1).
\]
Prediction steps leave $\Phi$ unchanged (an unrevealed prediction moves one unit from $r$ to $H$), and one binary refinement of $I$ decreases $\Phi$ by at most $\ln2$, because already-predicted points inside $I$ have been paid for in $H$ (Lemma~C.4). Revealing the label of an already-predicted point costs nothing in $\Phi$, because that point has been paid for in $H$; only unpredicted points inside the retained interval count in $r$.

\paragraph{One call contracts the potential by $O(1)$ (Lemma~C.5).} Conditional on any history before a call, the expected decrease of $\Phi$ during that call is at most $7\ln2$, whatever the query set. Call a state \emph{early} if fewer than $\lfloor m/2\rfloor$ of the $m$ points of $I$ have arrived. In an early state every point of the right child is unarrived, so its labels are fresh fair bits. If the query avoids the right child, the call stops when $J$ lies there, which has probability at least $1/2$; if it meets the right child, then with probability at least $\frac13\cdot\frac12$ the center is in the left child and some queried point of the right child is target-negative, which also stops the call. Hence each early refinement continues with probability at most $5/6$, the expected number of early refinements is at most $6$, and from the first late state on the remaining decrease is at most $\ln2$ since $r\le H+1$ there. Summing over calls, $\ln(n+1)-\mathbb E[\ln(H_{\rm final}+1)]\le7\ln2\cdot\mathbb E[Q]$ for an adaptive random number of calls, and Jensen's inequality gives $\mathbb E[H_{\rm final}]\ge(n+1)128^{-\mathbb E[Q]}-1$, hence Theorem~\ref{thm:C}. The fixed-instance clause follows because the distribution has finite support.

\paragraph{Why the uniform distribution cannot work.} Under a uniformly random order and cutoff (convention~E), predicting the first $m$ labels by posterior majority and then making one call on all $m$ observed labels achieves $\mathbb E[M]\le F(m)+2(T-m)/(m+2)$ for every consistent returned threshold, where $F(m)\le m/2$ is the exact zero-query optimum (Propositions~C.6--C.7); with $m=\lceil\sqrt T\rceil$ this is $O(\sqrt T)$ after one call, so no bound of the form $T\rho^{Q}$ with $Q=1$ can hold there. Moreover no history-uniform constant contraction bound holds on that distribution even with $+1$ regularization (Proposition~C.8). The hidden-center distribution above is designed so that each call decreases the paid potential by $O(1)$ in expectation, however many levels it descends.

\section{Oracle selection matters}
\label{sec:rule}

Theorem~\ref{thm:rule} is proved in Appendix~\ref{sec:proofAprime}, \S8; we describe the two rules and draw the interpretation.

\paragraph{The feasible-median rule admits deterministic logarithmic learning.} Index concepts by the number $k$ of positive instances; for every sample the feasible cutoffs form an integer interval $[\ell,u]$ (or are empty), and $\mathcal R_{\rm feas}$ returns the prefix of cutoff $\lfloor(\ell+u)/2\rfloor$. It depends only on $(S,\preceq)$ and knows nothing about the learner. The learner (Proposition~A.7) maintains $H_1\prec U\prec H_0$, where $H_1$ is known to be positive, $H_0$ known to be negative and $U$ the uncertain interval, initially all instances. While $|U|\ge2$ it queries $(H_1\times\{1\})\cup(H_0\times\{0\})$; the returned concept $h$ splits $U$ into $U_1=\{h=1\}$ and $U_0=\{h=0\}$ with $1\le|U_0|,|U_1|\le\lceil|U|/2\rceil$, because the feasible interval is exactly $[|H_1|,|H_1|+|U|]$ (under N with $H_1=\varnothing$ it is $[1,|U|]$). It predicts with $h$ until an error. An error with $h(x)=0$, $y=1$ certifies every point of $U_1$ positive and replaces $U$ by $U_0$; an error with $h(x)=1$, $y=0$ certifies $U_0$ negative and replaces $U$ by $U_1$. Each error halves $U$ (rounded up), so there are at most $\lceil\log_2T\rceil$ errors and at most as many calls. Past labels inside $U$ need not be included in the query: the interface allows arbitrary samples. At $T=8$ this learner has $M+Q\le6<T-\varepsilon$ on every instance, which refutes the universal extension of Theorem~\ref{thm:Aprime} with the required quantifiers.

\paragraph{The global-median rule does not.} Let $m=\lfloor(k_0+T)/2\rfloor$ with $k_0=\varepsilon$ and let $\mathcal R_{\rm global}$ return the feasible cutoff closest to $m$ (formally $\min\{u,\max\{\ell,m\}\}$, and $\bot$ on unrealizable samples). One empty query returns a prefix with $m$ positive instances; predicting with it until the first error and arbitrarily thereafter on the uncertain side gives $Q=1$ and $M\le\lceil T/2\rceil$, so this rule also violates the full bound $T-\varepsilon$. But it admits no deterministic $O(\log T)$-call, $O(\log T)$-mistake learner: fix an ordered block $H$ of $m$ instances, revealed to the learner and labeled positive, before a block $U$ of $T-m$ instances in unknown order. On such instances every call to $\mathcal R_{\rm global}$ can be simulated by at most one anchored $\Omin$ call on $U$ (whenever the sample has a positive in $U$ and no negative in $H$, every feasible cutoff exceeds $m$ and the rule returns the minimal one), so Theorem~\ref{thm:Aprime} on $U$ gives $M+Q\ge T-m\ge\lfloor T/2\rfloor$ on some instance.

\paragraph{Interpretation.} Three legal, memoryless, pre-declared rules for the same class and the same interface have qualitatively different deterministic complexities: the extremal rules have exact worst-case total cost $T-\varepsilon$; the global-median rule has linear total cost but not the full constant; the feasible-median rule admits $O(\log T)$ calls and mistakes. The randomized learner of AHR25 is oblivious to the rule. In this precise sense randomization substitutes for a favourable selection rule, and the selection rule is a complexity axis of oracle-based online learning, alongside the interface (Section~\ref{sec:cons}) and the class. We do not define a numerical measure of ``helpfulness'' of a rule; the one-pin common-value property of Section~\ref{sec:det} is a sufficient condition for the linear bound, and Section~\ref{sec:open} records what a characterization would have to explain. In particular, whether a rule returns an endpoint of the feasible interval does not decide the matter: neither median rule always returns an endpoint, yet their complexities differ.

\section{Consequences: fixed budgets and the weak consistency interface}
\label{sec:cons}

\paragraph{Deterministic tradeoffs at a fixed query budget.} Against adversarial oracles, one may ask for the exact minimax number of mistakes with at most $Q$ calls. We obtain bounds, exact values in the small- and large-budget regimes, and finite counterexamples to the natural conjecture that the lower bound is exact; the middle regime is open.

\begin{theorem}[Deterministic mistake--query bounds]
\label{thm:B}
Let $M^*_E(T,Q)$ (resp.\ $M^*_N$) be the minimax number of mistakes of deterministic learners making at most $Q$ ERM calls, against an adversarial oracle as in Theorem~\ref{thm:A}.
\begin{enumerate}[label=(\roman*),leftmargin=2em]
\item $\max\{T-Q,\lfloor\log_2(T+1)\rfloor\}\le M^*_E(T,Q)\le\max\{T-Q,\lceil T/2\rceil\}$; under N, $M^*_N(T,0)=M^*_N(T,1)=T-1$ and for $Q\ge1$, $\max\{T-Q,\lfloor\log_2T\rfloor\}\le M^*_N(T,Q)\le\min\{T-1,\max\{T-Q,\lceil T/2\rceil\}\}$. (The bound $T-Q$ under N is a budget statement for $Q\ge1$, proved by a first-action recurrence $M^*_N(n,q)\ge\min\{M^*_E(n-1,q-1),\,1+M^*_E(n-1,q),\,1+M^*_N(n-1,q)\}$ for $n\ge2$, $q\ge1$ and induction; it is not a pathwise statement.)
\item For $Q\le\lfloor T/2\rfloor$, $M^*_E(T,Q)=T-Q$; for $Q\ge2(T-1)$, $M^*_E(T,Q)=\lfloor\log_2(T+1)\rfloor$ and $M^*_N(T,Q)=\lfloor\log_2T\rfloor$ (this budget is sufficient, not necessary).
\item The lower bound in (i) is not the exact value: $M^*_E(5,3)=3$ and $M^*_N(6,4)=3$.
\end{enumerate}
\end{theorem}

The lower bounds combine Theorem~\ref{thm:A} with the known-order lower bound $\lfloor\log_2(T+1-\varepsilon)\rfloor$ (Lemma~D.10); the upper bounds come from a protection algorithm whose accounting is corrected in Appendix~\ref{sec:proofB}, and from sorting with $2(T-1)$ calls followed by majority prediction. The counterexamples are established analytically: a five-point game analysis under E, and a first-action recurrence reducing the N case to it (Appendix~\ref{sec:proofB}).

\paragraph{Interface contrast.} The exponential advantage of randomization at the ERM interface depends on a successful call returning an \emph{evaluable} concept. When the oracle returns only the realizability bit, the advantage disappears at the level of query order: both deterministic and randomized learners need $\Theta(T)$ calls for $O(\log T)$ mistakes.

\begin{theorem}[Weak consistency oracle]
\label{thm:D}
With only the weak consistency oracle, there is a deterministic learner that, on every instance, makes at most $66T$ calls and $\lfloor\log_2(T+1)\rfloor$ mistakes under E, and at most $67(T-1)$ calls and $\lfloor\log_2T\rfloor$ mistakes under N. Conversely, for $T-\varepsilon\ge100$: any learner---deterministic or randomized---that makes at most $(T-\varepsilon)/20$ calls on every run (a hard cap) incurs at least $(T-\varepsilon)/20$ expected mistakes on some instance (AHR25, Thm.~4.3, applied with effective horizon $T-\varepsilon$), and any learner with worst-instance expected mistakes $m$ and worst-instance expected calls $q$ satisfies $m+20q\ge(T-\varepsilon)/20$ (Lemma~E.4 in Appendix~\ref{sec:proofD}). Hence, at this interface, $\Theta(T)$ calls are necessary and sufficient for $O(\log T)$ mistakes for both deterministic and randomized learners.
\end{theorem}

\begin{remark}
\label{rem:D}
The randomized half of Theorem~\ref{thm:D} restates AHR25 (Thm.~4.5(2) and 4.3); the deterministic upper bound is a modest improvement over the $O(T\log T)$ of AHR25 (Thm.~4.5(1)) whose linear order can also be obtained by using deterministic linear-time selection \cite{BFPRT73} in place of a random pivot (this is a statement about this specific algorithm, not a claim that arbitrary randomized oracle algorithms can be derandomized for free), and its content is the explicit constants and the exact integer mistake bound; the selection procedure, the constants and the convention correction are in Appendix~\ref{sec:proofD}. Lemma~E.4 (a truncation lemma for expected budgets) is ours. The constant $128$ in Theorem~\ref{thm:C} is not optimized. The finite counterexamples in Theorem~\ref{thm:B} are established analytically.
\end{remark}

\section{Discussion and open problems}
\label{sec:open}

The results answer the question of AHR25 for thresholds on an unknown order, in the transductive protocol and with the consistency-type ERM interface, in the following sense: for a fixed natural oracle the deterministic query complexity of $O(\log T)$ mistakes is at least $T-\varepsilon-O(\log T)$, hence $\Theta(T)$, and the randomized one is $\Theta(\log T)$; the same linear bound holds against adversarial memoryless oracles; and the gap is governed by the oracle's selection rule. The general question, for arbitrary classes and in the non-transductive setting, remains open, as do the following.
\begin{itemize}
\item \textbf{Characterize favourable selection rules.} Determine which pre-declared ERM rules permit deterministic logarithmic-query, logarithmic-mistake learning. The one-pin common-value property (Section~\ref{sec:det}) is a sufficient condition for a linear lower bound, satisfied by the extremal rules; endpoint versus non-endpoint selection alone does not characterize the complexity, since neither the feasible-median rule nor the global-median rule always returns an endpoint of the feasible interval, yet the former admits deterministic $O(\log T)$ calls and mistakes while the latter forces linear total cost (Theorem~\ref{thm:rule}). A parameter measuring how much a rule can shrink the feasible set in one call, related to the deterministic mistake--query complexity even one-sidedly, would turn the three examples into a theorem.
\item \textbf{The middle regime.} Determine $M^*_E(T,Q)$ for $\lfloor T/2\rfloor<Q<2(T-1)$; the value is not $\max(T-Q,\lfloor\log_2(T+1)\rfloor)$ (Theorem~\ref{thm:B}(iii)).
\item \textbf{Constants.} Improve the constant $128$ in Theorem~\ref{thm:C}.
\item \textbf{Beyond thresholds.} Extend the separation to arbitrary Littlestone classes and to the non-transductive setting, which is the general form of the open question of AHR25. The charging and fixation steps of Section~\ref{sec:det} and the paid-potential device of Section~\ref{sec:rand} are not threshold-specific; the one-pin lemma and the interval-tree distribution are.
\end{itemize}

\appendix

\section{Proof of Theorem~\ref{thm:Aprime} and Theorem~\ref{thm:rule}}
\label{sec:proofAprime}
\subsection*{1. Statement and conventions}

All lemmas below are proved here. We use the consistency-ERM interface of Section 2, including full-vector returns and unrestricted finite query samples. No weak-consistency interface is used.

For a finite ordered domain, write:

\begin{itemize}
\item \(\mathcal O_{\min}\): return the smallest consistent prefix, or \(\bot\) if none exists;
\item \(\mathcal O_{\max}\): return the largest consistent prefix, or \(\bot\) if none exists.
\end{itemize}

Under convention \textbf{E}, the domain is
\[
X=\{s,x_1,\ldots,x_T\},
\qquad s\prec x\quad(x\ne s).
\]
Thus every concept labels \(s\) by \(1\), and the smallest prefix is \(\{s\}\). This is the rule fixed in Theorem~\ref{thm:Aprime}.

Under convention \textbf{N}, the domain is \(X=\{x_1,\ldots,x_T\}\), and prefixes must be nonempty. In particular, when a query has no positive labels, \(\mathcal O_{\min}\) returns the singleton containing the order-minimum if that singleton is consistent, and otherwise returns \(\bot\). This is the appropriate nonempty-prefix version of the same rule.

\subsubsection*{Statement of Theorem~\ref{thm:Aprime} (fixed extremal ERM rules)}

Fix either \(\mathcal R=\mathcal O_{\min}\) or \(\mathcal R=\mathcal O_{\max}\), publicly and before choosing the learner or instance. For every deterministic learner \(\mathcal A\), and every prescribed sequence of distinct instances \(x_1,\ldots,x_T\), there exist a total order \(\preceq^\star\) and an endpoint \(z^\star\in X\) such that
\[
M\!\left(\mathcal A,\mathcal R(\cdot,\preceq^\star),
                   x_{1:T},c_{z^\star}\right)
+
Q\!\left(\mathcal A,\mathcal R(\cdot,\preceq^\star),
                   x_{1:T},c_{z^\star}\right)
\ge T-\varepsilon,
\]
where \(\varepsilon=0\) under E and \(\varepsilon=1\) under N.

The oracle is the originally declared function \(\mathcal R\); no return-value table is selected after constructing the instance.

The statement also holds for every predeclared sample-dependent extremal selector
\[
\mathcal R_\eta(S,\preceq)=
\begin{cases}
\mathcal O_{\min}(S,\preceq),&\eta(S)=0,\\
\mathcal O_{\max}(S,\preceq),&\eta(S)=1,
\end{cases}
\tag{A.1}
\]
where \(\eta\) depends on \(S\), not on the unknown order or interaction history.

We first prove the theorem for \(\mathcal O_{\min}\).

\medskip

\subsection*{2. Anchored states}

Use a distinguished anchor \(\sigma\), which is constrained to be the order-minimum and therefore always has target label \(1\).

\begin{itemize}
\item Under E, take \(\sigma=s\). It is not an online instance.
\item Under N, take \(\sigma=x_1\). It is an online instance whose label is fixed to \(1\).
\end{itemize}

Let
\[
V=X\setminus\{\sigma\},\qquad n=|V|=T-\varepsilon.
\]

The construction maintains ordered lists \(L,R\) and an unordered set \(F\), partitioning \(V\). Its retained orders are
\[
\Omega(L,F,R)
=
\left\{
  \sigma\cdot L\cdot\pi(F)\cdot R:
  \pi(F)\text{ is any permutation of }F
\right\}.
\tag{A.2}
\]

The retained order--target pairs are
\[
\Gamma(L,F,R)
=
\left\{
\left(
 \sigma\cdot L\cdot\pi(F)\cdot R,\,
 \mathbf 1_{\{\sigma\}\cup L\cup\{\pi_1,\ldots,\pi_k\}}
\right):
0\le k\le |F|
\right\}.
\tag{A.3}
\]

Thus:

\begin{itemize}
\item every point of \(L\) is committed to target label \(1\);
\item every point of \(R\) is committed to target label \(0\);
\item the target may cut anywhere through the freely permutable block \(F\).
\end{itemize}

Commitment does not mean that the label has already been revealed. In particular, query-pinned points can be assigned their eventual target labels immediately. Hence no separate blocks of pinned points are needed: pinned points are absorbed directly into \(L\) or \(R\).

Initially,
\[
L=R=\varnothing,\qquad F=V.
\]

\subsubsection*{Endpoint refinements}

There are two permitted one-point refinements:

\[
\begin{array}{ll}
\text{left pin of }p\in F:
 &(L,F,R)\longmapsto(L\cdot p,F\setminus\{p\},R),\\[1mm]
\text{right pin of }p\in F:
 &(L,F,R)\longmapsto(L,F\setminus\{p\},p\cdot R).
\end{array}
\tag{A.4}
\]

Both preserve nonemptiness and satisfy
\[
\Gamma_{\mathrm{new}}\subseteq\Gamma_{\mathrm{old}},
\qquad
\Omega_{\mathrm{new}}\subseteq\Omega_{\mathrm{old}}.
\tag{A.5}
\]

Indeed, a left pin restricts the old free permutation to \(p\cdot\pi(F\setminus\{p\})\) and restricts the target to include \(p\). A right pin restricts it to \(\pi(F\setminus\{p\})\cdot p\) and restricts the target to exclude \(p\). These were already allowed choices in (A.3).

No relative order among the remaining free points is imposed.

\medskip

\subsection*{3. Prefix realizability and the common-value query lemma}

\subsubsection*{Lemma A.1 --- Prefix criterion}

Suppose \(\sigma\) is the minimum. Reject immediately if a sample \(S\) is contradictory or contains \((\sigma,0)\). Otherwise define
\[
A=\{x\ne\sigma:(x,1)\in S\},
\qquad
B=\{x:(x,0)\in S\}.
\]

Then \(S\) is realizable if and only if there are no \(b\in B,a\in A\) with \(b\prec a\). When realizable,
\[
\mathcal O_{\min}(S,\preceq)=
\begin{cases}
\mathbf 1_{\{\sigma\}},&A=\varnothing,\\
\mathbf 1_{\{x:x\preceq\max_\preceq A\}},&A\ne\varnothing.
\end{cases}
\tag{A.6}
\]

\textbf{Proof.} A prefix containing \(a\) must contain every \(b\prec a\), proving necessity. Conversely, if \(A\ne\varnothing\), the prefix through \(\max A\) contains all positives and no negatives. If \(A=\varnothing\), the singleton anchor is consistent. These are the smallest possible consistent prefixes. \qed

For a current state, write
\[
b\triangleleft a
\quad\Longleftrightarrow\quad
b\prec a\text{ in every order in }\Omega(L,F,R).
\tag{A.7}
\]
This relation is determined by the block order and the fixed internal orders of \(L,R\). Two distinct free points are incomparable under \(\triangleleft\).

\subsubsection*{Lemma A.2 --- One-pin common-value query lemma}

For every state and every query \(S\), the construction can make at most one endpoint pin so that
\[
\mathcal O_{\min}(S,\preceq)
\]
has exactly the same full-vector value for every order retained after the pin.

\textbf{Proof.}

Contradictory samples and samples containing \((\sigma,0)\) have the common answer \(\bot\), with no refinement. Otherwise use \(A,B\) from Lemma A.1.

\paragraph{Case I: an inversion is already forced}

If some \(b\in B,a\in A\) satisfy \(b\triangleleft a\), answer \(\bot\) without changing the state.

Lemma A.1 makes this the required answer for every retained order.

\paragraph{Case II: both label sides contain free points}

Suppose Case I does not apply and
\[
A\cap F\ne\varnothing,\qquad B\cap F\ne\varnothing.
\]

Choose \(b\in B\cap F\), using the smallest arrival index to make the construction explicit, and left-pin \(b\).

Because the sample is not contradictory, any \(a\in A\cap F\) is distinct from \(b\). Every new retained order has \(b\prec a\). Therefore its prescribed oracle answer is \(\bot\).

Exactly one free point was pinned.

\paragraph{Remaining cases}

Now suppose neither Case I nor Case II applies.

Every positive--negative pair has at least one endpoint outside \(F\), so its relative order is already fixed. Absence of a forced inversion therefore implies that the sample is realizable in every current retained order. It remains to stabilize the \emph{selected smallest prefix}.

There are three exhaustive possibilities.

\textbf{Case III-R: \(A\cap R\ne\varnothing\).}

Let \(p\) be the last point of \(A\cap R\) in the fixed list \(R\). In every retained order,
\[
p=\max_\preceq A.
\]
Its down-set is the fixed set
\[
D=\{\sigma\}\cup L\cup F\cup R_{\le p},
\tag{A.8}
\]
where \(R_{\le p}\) is the initial segment of \(R\) ending at \(p\).

Answer \(\mathbf1_D\), without a pin.

This case includes samples having positives in both \(F\) and \(R\).

\textbf{Case III-F: \(A\cap R=\varnothing\) and \(A\cap F\ne\varnothing\).}

Case II's exclusion gives \(B\cap F=\varnothing\). Case I's exclusion gives \(B\cap L=\varnothing\), since every left point precedes every free positive. Thus
\[
B\subseteq R.
\tag{A.9}
\]

Choose \(p\in A\cap F\) by smallest arrival index and right-pin it:
\[
R_{\mathrm{new}}=p\cdot R_{\mathrm{old}}.
\]

The point \(p\) is after every remaining free point but \textbf{before every old right point}. Since there are no positives in the old right block,
\[
p=\max_\preceq A
\quad\text{for every new retained order}.
\]
All negatives are in the old right block and hence follow \(p\). The required answer is therefore
\[
\mathbf1_D,\qquad
D=\{\sigma\}\cup L\cup F_{\mathrm{old}}.
\tag{A.10}
\]

This is one common full vector, obtained with one pin.

\textbf{Case III-L: \(A\subseteq L\).}

If \(A=\varnothing\), answer \(\mathbf1_{\{\sigma\}}\). Otherwise let \(p\) be the last point of \(A\) in \(L\), and answer
\[
\mathbf1_{\{\sigma\}\cup L_{\le p}}.
\tag{A.11}
\]

No pin is needed.

These cases exhaust all queries and establish the assertion. \qed

\medskip

\subsection*{4. Labels, invariants, and charging}

\subsubsection*{Lemma A.3 --- Strong transcript invariant}

The construction can maintain all of the following:

\begin{enumerate}
\item \(\Gamma(L,F,R)\) is nonempty.
\item Every retained order--target pair agrees with every label already committed by the construction.
\item Every previous oracle answer is exactly the value of the predeclared \(\mathcal O_{\min}\) on \textbf{every} order in the current \(\Omega(L,F,R)\).
\item Every point still in \(F\) has not yet been predicted.
\end{enumerate}

At a prediction on a free point, the construction can force an error while preserving these properties.

\textbf{Proof.}

The assertions hold initially. Queries are handled by Lemma A.2. Their refinements preserve nonemptiness and nesting by (A.5). Consequently all earlier common-value assertions remain true, and the new answer has the same property.

At a prediction event:

\begin{itemize}
\item if \(x=\sigma\) or \(x\in L\), commit \(y=1\);
\item if \(x\in R\), commit \(y=0\);
\item if \(x\in F\), observe the deterministic prediction \(\widehat y\), set
\end{itemize}
  \[
  y=1-\widehat y,
  \]
  and left-pin \(x\) when \(y=1\), or right-pin \(x\) when \(y=0\).

The last operation is an endpoint refinement. It leaves a nonempty subset of previously retained order--target pairs, all having the newly committed label. It also incurs an error and removes the predicted point from \(F\).

A label may be committed in this offline simulation at the prediction event. Lemma A.5 below turns the result into an actual oblivious instance, so the final interaction obeys the label-before-prediction requirement of the transductive protocol. \qed

\subsubsection*{Lemma A.4 --- Additive charging}

If all \(T\) predictions are completed, then
\[
M+Q\ge n=T-\varepsilon.
\tag{A.12}
\]

\textbf{Proof.}

Let \(p\) be the number of points removed from \(F\) by queries. Lemma A.2 gives
\[
p\le Q.
\]
Every other initially free point is removed at its prediction event and incurs an error. Since every point in \(V\) eventually appears,
\[
M\ge n-p.
\]
Adding the inequalities gives (A.12).

Equivalently, the initial potential is \(\Phi_0=|F|=n\); a query decreases it by at most one, and a prediction decreases it only when an error occurs. At completion, \(\Phi=0\). \qed

\medskip

\subsection*{5. Oblivious fixation and memorylessness}

\subsubsection*{Lemma A.5 --- Finite fixation and replay}

The construction yields a single fixed order and target with the claimed cost, using the originally declared memoryless oracle.

\textbf{Proof.}

If \(n=0\), the assertion is immediate. Otherwise simulate the deterministic learner, stopping at the first of:

\begin{enumerate}
\item completion of all \(T\) predictions; or
\item completion of the \(n\)-th query.
\end{enumerate}

This uses at most \(n\) query events and \(T\) prediction events. In particular, a learner that would make infinitely many queries is covered by the second stopping condition.

Let the final retained state be \((L,F,R)\). Order the remaining \(F\) by arrival index, obtaining \(\pi_0(F)\), and define
\[
\preceq^\star
=
\sigma\cdot L\cdot\pi_0(F)\cdot R,
\tag{A.13}
\]
and
\[
z^\star=
\begin{cases}
\text{last point of }L,&L\ne\varnothing,\\
\sigma,&L=\varnothing.
\end{cases}
\tag{A.14}
\]

This is a retained order--target pair: it chooses the cut immediately after \(L\). Hence its target agrees with all committed labels.

Now run the learner from the beginning on this fixed order, target, and the fixed oracle
\[
S\longmapsto\mathcal O_{\min}(S,\preceq^\star).
\]

Induct over the recorded events. Determinism implies that the next query or prediction is unchanged whenever the preceding transcript is unchanged. For every recorded query \(S\), Lemma A.3 gives
\[
\text{recorded answer}
=
\mathcal O_{\min}(S,\preceq^\star),
\]
including the entire returned vector. Recorded labels also agree with \(c_{z^\star}\). Thus the interaction replays exactly through the stopping event.

If all predictions were completed, Lemma A.4 applies. Otherwise the replay already contains \(n\) queries, so \(Q\ge n\).

The fixed oracle is legal on \textbf{every} sample by its definition and Lemma A.1, including samples outside the recorded transcript. No answer cache or oracle-extension table is needed.

Repeated queries are automatically consistent: every later retained family is a nonempty subset of the earlier one, and a deterministic function cannot have two different values on the same sample and the same retained order.

Finally, (A.13)--(A.14) are selected before the actual replay. The actual adversary is therefore oblivious as required in Section 2. \qed

Lemmas A.1--A.5 prove Theorem~\ref{thm:Aprime} for \(\mathcal O_{\min}\).

\medskip

\subsection*{6. The maximal-prefix rule and other extremal selectors}

\subsubsection*{Lemma A.6 --- Anchored reversal-complement duality}

Keep the anchor fixed. Reverse the order of \(V\), complement every label on \(V\), and leave the anchor label unchanged.

Denote this transformation by \(D\). For a concept,
\[
(Dh)(\sigma)=1,\qquad (Dh)(x)=1-h(x)\quad(x\in V).
\]
For samples, labels on \(V\) are complemented, while labels on \(\sigma\) are unchanged; let \(D\bot=\bot\).

Then
\[
D\!\left(\mathcal O_{\max}(S,\preceq)\right)
=
\mathcal O_{\min}(DS,D\preceq).
\tag{A.15}
\]

\textbf{Proof.}

An anchored prefix containing \(k\) points of \(V\) becomes an anchored prefix containing \(n-k\) points in the reversed order. This is a consistency-preserving bijection between the two concept classes, reversing inclusion. Consequently, the largest consistent prefix becomes the smallest consistent prefix. Unrealizable samples remain unrealizable. \qed

Given a deterministic learner for \(\mathcal O_{\max}\), simulate it using \(\mathcal O_{\min}\) through (A.15), complementing predictions and revealed labels on \(V\), and leaving the anchor unchanged. Queries and errors are preserved exactly. Applying the already proved anchored minimum-rule theorem proves Theorem~\ref{thm:Aprime} for the maximum rule.

This also handles N correctly: the actual minimum instance is kept as an always-positive anchor. Ordinary reversal and complementation of \textbf{all} points would incorrectly introduce an empty concept.

For completeness, duality maps a state to
\[
D(L,F,R)=(\operatorname{reverse}R,\ F,\ \operatorname{reverse}L).
\tag{A.16}
\]
Thus the maximum rule has the same one-pin common-value property. For the selector (A.1), \(\eta(S)\) is fixed independently of the compatible order, so one may apply the minimum or maximum one-pin procedure separately at each query. The remainder of the proof is unchanged.

More generally, the proof applies to any fixed legal rule satisfying the following sufficient condition:

\begin{quote}\textbf{One-pin common-value property:} On every anchored block state and every sample, some refinement by zero or one endpoint pin makes the rule's full answer constant on all retained orders.\end{quote}

This is a sufficient condition, not a claimed characterization of every possible rule.

\medskip

\subsection*{7. Sharpness and relation to AHR25}

The total-cost lower bound is exact:
\[
\inf_{\mathcal A\ {\rm deterministic}}
\sup_{\preceq,z}
\bigl(M_{\mathcal A}+Q_{\mathcal A}\bigr)
=
T-\varepsilon
\tag{A.17}
\]
for either extremal rule.

The matching upper bounds require no oracle calls:

\begin{itemize}
\item under E, any fixed predictions make at most \(T\) errors;
\item under N, always predicting \(1\) makes at most \(T-1\) errors, since every nonempty prefix labels at least one instance by \(1\).
\end{itemize}

Thus (A.17) includes both an upper-bound algorithm and the lower bound; it is not an assertion about an unproved exact fixed-\(Q\) tradeoff.

Our common-value lemma establishes a one-point property for genuine total-order thresholds and a fixed natural ERM selection rule; it does \textbf{not} assert that the returned vector contains only one bit. Theorem~4.4 of AHR25 is not invoked here; its statement, and why its hypothesis fails for the threshold family, are discussed once in Appendix~\ref{sec:proofA}.

By Theorem~\ref{thm:Aprime}, a deterministic learner with \(O(\log T)\) worst-case errors under \(\mathcal O_{\min}\) must make
\[
T-\varepsilon-O(\log T)
\]
queries on some fixed instance. AHR25 (Thm.~4.5(4)) supplies the randomized \(O(\log T)\)-error, \(O(\log T)\)-expected-query upper bound under the same fixed rule, since that algorithm works for arbitrary legal ERM returns. This establishes the intended separation without adversarial tie-breaking.

\medskip

\subsection*{8. Arbitrary predeclared rules: a counterexample}

The universal extension
\[
\text{``Theorem~\ref{thm:Aprime} holds for every fixed function of }(S,\preceq)\text{''}
\]
is false.

\subsubsection*{8.1 Median of the feasible prefixes}

Index concepts by the number \(k\) of positive online-domain points:

\begin{itemize}
\item E: \(k\in\{0,\ldots,T\}\), with the sentinel always positive;
\item N: \(k\in\{1,\ldots,T\}\).
\end{itemize}

For any sample, its feasible cutoff indices form an integer interval
\[
I(S,\preceq)=\{\ell,\ell+1,\ldots,u\},
\]
or are empty. Define the predeclared rule
\[
\mathcal R_{\mathrm{feas}}(S,\preceq)=
\begin{cases}
\bot,&I(S,\preceq)=\varnothing,\\
\text{prefix of cutoff }\left\lfloor\dfrac{\ell+u}{2}\right\rfloor,
&\text{otherwise}.
\end{cases}
\tag{A.18}
\]

This rule depends only on \(S\) and the order. In particular, it does not know a learner's state or free block.

\subsubsection*{Proposition A.7 --- Deterministic logarithmic learning under (A.18)}

Under either convention, for \(T\ge2\), there is a deterministic learner satisfying, for every order and target,
\[
Q\le\lceil\log_2T\rceil,\qquad
M\le\lceil\log_2T\rceil.
\tag{A.19}
\]

\textbf{Proof.}

Maintain a decomposition of the instance set into
\[
H_1\prec U\prec H_0,
\]
where \(H_1\) is an order-prefix known to have target label \(1\), \(H_0\) is an order-suffix known to have target label \(0\), and \(U\) is the intervening interval. Initially \(U\) is the entire instance set.

For \(|U|\ge2\), query
\[
S=(H_1\times\{1\})\cup(H_0\times\{0\}).
\tag{A.20}
\]
This sample is consistent. Let \(h\) be the returned concept and set
\[
U_1=\{x\in U:h(x)=1\},\qquad U_0=U\setminus U_1.
\]

Write \(a=|H_1|\), \(m=|U|\). Under E, the feasible cutoff interval is exactly \([a,a+m]\), so \(|U_1|=\lfloor m/2\rfloor\). Under N the same holds when \(a>0\); when \(a=0\), the feasible interval is \([1,m]\), so \(|U_1|=\lceil m/2\rceil\). Therefore
\[
1\le |U_0|,|U_1|\le\lceil m/2\rceil
\qquad(m\ge2).
\tag{A.21}
\]

Predict with \(h\) until an error occurs.

\begin{itemize}
\item If \(h(x)=0\) but \(y=1\), then every point in \(U_1\) precedes \(x\) and must have target label \(1\). Replace
\end{itemize}
  \[
  H_1\leftarrow H_1\cup U_1,\qquad U\leftarrow U_0.
  \]
\begin{itemize}
\item If \(h(x)=1\) but \(y=0\), then every point in \(U_0\) follows \(x\) and must have target label \(0\). Replace
\end{itemize}
  \[
  H_0\leftarrow U_0\cup H_0,\qquad U\leftarrow U_1.
  \]

Thus each error reduces the active interval to at most half, rounded up. If the new interval is a singleton, it consists of the just-mistaken point, whose label is now known; all target labels are then determined. Otherwise make the next query before the next prediction.

After \(j\) such errors,
\[
|U|\le \left\lceil\frac{T}{2^j}\right\rceil.
\]
Consequently there are at most \(\lceil\log_2T\rceil\) errors and at most that many queries. For \(T=1\), no queries and at most one error suffice under E; under N, predict \(1\) without error.

Past labels inside \(U\) need not be included in (A.20). The interface of Section 2 allows arbitrary samples; a queried hypothesis is not required to fit observations omitted from that sample. \qed

At \(T=8\), this gives one explicit deterministic learner with
\[
M+Q\le6
\]
for \textbf{every} order and target, under both conventions. Since \(6<8\) under E and \(6<7\) under N, the all-rules extension is refuted with the required quantifiers.

These are rule-specific upper bounds. No balance is assumed of an arbitrary consistency-ERM return.

\subsubsection*{8.2 The different ``closest to the global median'' rule}

A different natural rule chooses the feasible cutoff closest to a fixed global median. This is different from (A.18).

Let
\[
k_0=\begin{cases}0,&E,\\1,&N,\end{cases}
\qquad
m=\left\lfloor\frac{k_0+T}{2}\right\rfloor.
\]
Define \(\mathcal R_{\mathrm{global}}(S,\preceq)=\bot\) when \(I(S,\preceq)=\varnothing\) (this covers contradictory samples and, under E, a negative sentinel); otherwise, when the feasible interval is \([\ell,u]\), let it return the prefix of cutoff
\[
\operatorname{proj}_{[\ell,u]}(m)=\min\{u,\max\{\ell,m\}\}.
\tag{A.22}
\]

This rule also fails the full inequality of Theorem~\ref{thm:Aprime}(a), but \textbf{does not} admit simultaneous \(O(\log T)\) mistakes and queries.

\paragraph{Upper bound sufficient to refute the bound of Theorem~\ref{thm:Aprime}(a)}

One empty-sample query returns a prefix having \(m\) positive instances. Predict with that concept until the first error.

If an error occurs, the target uncertainty is confined to the side containing the mistaken point: the opposite side has a forced target label. Thereafter, arbitrary predictions inside the uncertain side incur at most one error per remaining point.

Hence
\[
Q=1,\qquad M\le\max(m,T-m)=\lceil T/2\rceil.
\tag{A.23}
\]
This argument is valid for any returned threshold, with its actual side sizes. Balance enters only through the declared rule's empty-query output.

For \(T=8\), (A.23) gives \(M+Q\le5\), refuting the bound of Theorem~\ref{thm:Aprime}(a) under both conventions.

\paragraph{Linear lower bound for the global-median rule}

For every deterministic learner, some instance satisfies
\[
M+Q\ge T-m.
\tag{A.24}
\]

To prove this, let \(r=T-m\). If \(r=0\) the claim is \(M+Q\ge0\) and holds trivially, so assume \(r\ge1\). If some instance in the family constructed below makes the original learner issue at least \(r\) queries (possibly infinitely many), then (A.24) already holds on that instance; otherwise the original learner makes at most \(r-1\) queries on every instance of the family, and the per-query simulation below yields a derived learner that completes the protocol. Fix an ordered block \(H\) of \(m\) instances before a block \(U\) of \(T-m\) instances whose internal order is unknown. Restrict targets to label all of \(H\) positively, with an arbitrary, possibly empty, prefix target on \(U\). The order of \(H\) can even be revealed to the learner.

An interaction with \(\mathcal R_{\mathrm{global}}\) can be simulated using an anchored \(\mathcal O_{\min}\) oracle on \(U\), with at most one minimum-oracle call per original query:

\begin{enumerate}
\item Contradictions and an illegal negative sentinel are rejected directly.
\item If the sample has a negative in \(H\) and a positive in \(U\), it is unrealizable.
\item If it has a negative in \(H\) but no positive in \(U\), any feasible cutoff lies strictly before \(m\). Its largest feasible cutoff, or infeasibility, is computable entirely inside the known ordered block \(H\).
\item If it has neither a negative in \(H\) nor a positive in \(U\), cutoff \(m\) is feasible and is returned.
\item Otherwise it has a positive in \(U\) and no negative in \(H\). Every feasible cutoff exceeds \(m\), so (A.22) selects the \textbf{minimum} feasible cutoff. One minimum-oracle query on the restriction of the sample to \(U\) produces exactly that answer; extend it by assigning \(1\) to all of \(H\).
\end{enumerate}

Simulate predictions on \(H\) using their known labels, and use the learner's predictions on \(U\) as those of the derived minimum-oracle learner. The derived learner has no more queries and no more mistakes than the original learner.

Applying the already proved E-convention minimum-rule theorem to \(U\), with a virtual anchor, proves (A.24). Under N, the nonempty block \(H\) supplies the anchor when embedding the resulting target back into the original instance.

Since \(T-m\ge\lfloor T/2\rfloor\), a deterministic learner cannot simultaneously achieve \(O(\log T)\) mistakes and \(O(\log T)\) queries under (A.22).

Thus the precise conclusions are:

\begin{center}\small\begin{tabular}{p{0.45\linewidth}p{0.45\linewidth}}\hline
Predeclared rule & Deterministic conclusion \\ \hline
Minimum / maximum / sample-dependent extremal selectors & Exact worst-case minimax total cost \(T-\varepsilon\) \\
Closest feasible cutoff to a fixed global median & Linear total-cost lower bound; full \(T-\varepsilon\) bound is false \\
Median of the feasible cutoffs & Deterministic \(O(\log T)\) queries and mistakes \\
\hline\end{tabular}\end{center}

The general-rule counterexamples do not refute Theorem~\ref{thm:A} or Theorem~\ref{thm:Aprime}: they change the publicly declared oracle rule.

\medskip

\section{Proof of Theorem~\ref{thm:A}}
\label{sec:proofA}
\subsection*{Theorem~\ref{thm:A}: an additive deterministic lower bound with a fixed memoryless ERM oracle}

\textbf{Conventions and quantifiers.} Let \(T\ge1\), and announce the distinct instances
\[
I_T=\{x_1,\ldots,x_T\}
\]
in the fixed arrival sequence \((x_1,\ldots,x_T)\). Consider either of the following conventions (Section 2):

\begin{enumerate}
\item \textbf{Empty-prefix convention:} \(X=I_T\cup\{s\}\), where \(s\) is placed before every instance. Put \(\varepsilon=0\).
\item \textbf{Nonempty-prefix convention:} \(X=I_T\). Put \(\varepsilon=1\).
\end{enumerate}

For every deterministic learner \(\mathcal A\), there exist a total order \(\preceq^*\) on \(X\), a point \(z^*\in X\), and a function
\[
O^*:2^{X\times\{0,1\}}
   \longrightarrow C_{\preceq^*}\cup\{\bot\}
\]
such that:

\begin{itemize}
\item \(O^*\) is a legal consistency-type ERM oracle on \textbf{every} sample \(S\);
\item \(O^*(S)\) depends only on \(S\), not on the time, query history, or target;
\item the order, target, instance sequence, and oracle function are all fixed before the actual interaction; and
\item on this one fixed instance,
\end{itemize}
  \[
  \boxed{M(\mathcal A)+Q(\mathcal A)\ge T-\varepsilon.}
  \]

Here the learner may submit arbitrary samples, including contradictory samples, samples involving previously labeled or future instances, and samples involving the non-instance \(s\). It may query at arbitrary times and receives complete concept label vectors, as stipulated in the model of Section~2.

Thus the additive constant is \(0\) under the empty-prefix convention and \(1\) under the nonempty-prefix convention.

The proof is self-contained below.

\medskip

\subsection*{1. A common anchored formulation}

For both conventions, introduce an anchor \(a\) and a set \(V\) of initially free instances:

\[
\begin{array}{c|c|c|c}
\text{Convention}&a&V&n:=|V|\\ \hline
\text{Empty prefix}&s&I_T&T\\
\text{Nonempty prefix}&x_1&I_T\setminus\{x_1\}&T-1
\end{array}
\]

We commit to placing \(a\) first in the total order. Every concept consequently labels \(a\) by \(1\).

In the nonempty-prefix convention, choosing \(x_1\) as the least element is simply an allowed adversarial choice of order. The argument would remain valid even if this information were given to the learner for free.

A state consists of:

\begin{itemize}
\item an ordered list \(L\);
\item an unordered set \(F\);
\item an ordered list \(R\);
\end{itemize}

partitioning \(V\). Its compatible orders are
\[
\mathcal E(L,F,R)
 =
 \bigl\{\,a\cdot L\cdot\pi(F)\cdot R:
             \pi(F)\text{ is a permutation of }F\,\bigr\}.
\tag{B.1}
\]

Points in \(L\) are committed to target label \(1\), and points in \(R\) to target label \(0\). They need not already have arrived. Points in \(F\) have not arrived and remain free.

We use a list also for its underlying set when taking unions. For a nonempty \(H\subseteq X\), write \(h_H=\mathbf 1_H\).

All lemmas below apply to \textbf{both conventions}, with \(n=T-\varepsilon\).

\subsubsection*{Lemma B.1 --- Prefix consistency criterion}

Suppose a total order \(\tau\) has \(a\) as its least element. For a sample \(S\), define
\[
P(S)=\{x:(x,1)\in S\},\qquad
N(S)=\{x:(x,0)\in S\}.
\]

If \(P(S)\cap N(S)\ne\varnothing\), or if \(a\in N(S)\), then \(S\) is not realizable by \(C_\tau\).

Otherwise put
\[
A=P(S)\setminus\{a\},\qquad B=N(S).
\]
Then \(S\) is realizable if and only if there is no pair \(b\in B,a'\in A\) with
\[
b\prec_\tau a'.
\tag{B.2}
\]

\textbf{Proof.} Contradictory labels cannot be realized (Section 2). Every nonempty prefix contains the least point \(a\), so a negative label on \(a\) is also impossible.

An inversion as in (B.2) prevents any prefix from containing \(a'\) while excluding \(b\). Conversely, if there is no inversion and \(A\ne\varnothing\), the prefix ending at \(\max_\tau A\) contains every positive sample point and no negative sample point. If \(A=\varnothing\), the singleton prefix \(\{a\}\) is consistent. \qed

Notice that the oracle tests realizability by the \textbf{class}, not agreement with the target. For example, a positive singleton query remains realizable even when that point has already received target label \(0\).

\medskip

\subsection*{2. The one-pin query rule}

The adversary keeps an initially empty table \(\mathsf{Ans}\), keyed by the entire sample set \(S\).

The partial order \(\prec_K\) of a state \(K=(L,F,R)\) fixes the internal orders of \(L,R\), puts \(L\) before \(F\) before \(R\), and leaves distinct points of \(F\) incomparable.

A useful representation is the ordered list of \textbf{atoms}
\[
W_K=(\ell_1,\ldots,\ell_p,\,[F],\,r_1,\ldots,r_q),
\tag{B.3}
\]
omitting \([F]\) when \(F=\varnothing\). The entire free set is one atom.

\subsubsection*{Lemma B.2 --- Exhaustive query handling with at most one pin}

For every state \(K\) and sample \(S\), the following rule either returns an answer already cached, or constructs a new answer legal for every order in the resulting state. A fresh query removes at most one point from \(F\).

\textbf{Rule and proof.}

\textbf{Cached sample.} If \(S\in\operatorname{dom}(\mathsf{Ans})\), return its stored answer without changing the state. Preservation of this answer is proved in Lemma B.3.

\emph{Why this rule needs the cache.} Let initially \(F=\{u,v\}\) and let the first query be \(\{(u,1)\}\); case (iii) below may answer with the vector that is positive on both \(u\) and \(v\). Suppose \(u\) then arrives and is labelled \(1\), so that the state becomes \(L=[u]\), \(F=\{v\}\). If the same sample \(\{(u,1)\}\) were now treated as a fresh query, the rule may legally return the prefix \([u]\) alone. Both answers are legal, but the same sample would receive two different full vectors, so without the first-answer cache the transcript could not be frozen into a function of the sample alone (Lemma B.5). This is a property of the particular adaptive answering rule used here; the fixed-rule proof in Appendix~\ref{sec:proofAprime} needs no answer cache.

For a fresh sample, first handle contradictions or a negative label on \(a\) by returning \(\bot\), without changing the state. Otherwise strip the redundant positive label on \(a\), obtaining disjoint \(A,B\subseteq V\).

There are exactly three remaining cases.

\subsubsection*{(i) A forced inversion already exists}

If some \(b\in B,a'\in A\) satisfy
\[
b\prec_K a',
\]
return \(\bot\), with no state change.

Every order in \(\mathcal E(K)\) contains this inversion, so Lemma B.1 proves the answer legal.

\subsubsection*{(ii) Both labels occur in the free block}

Suppose there is no forced inversion, but
\[
A\cap F\ne\varnothing,\qquad B\cap F\ne\varnothing.
\]
Choose one \(b\in B\cap F\), breaking ties by the public arrival indices, and update
\[
L'=L\cdot b,\qquad F'=F\setminus\{b\},\qquad R'=R.
\tag{B.4}
\]
Return \(\bot\).

Because \(A,B\) are disjoint, some \(a'\in A\cap F\) is different from \(b\). Every new compatible order places \(b\) before this \(a'\). Hence the sample is unrealizable in every new compatible order.

Moreover,
\[
\mathcal E(L',F',R')\subseteq\mathcal E(L,F,R):
\]
the new orders are exactly old-compatible orders whose free permutation begins with \(b\). Exactly one free point is removed.

Committing the target label of \(b\) to \(1\), although \(S\) asks for label \(0\) there, is legitimate: the oracle is answering that \(S\) is \textbf{not} realizable by the chosen class. It is not required to make the target consistent with \(S\).

\subsubsection*{(iii) No forced inversion and no mixed-label free block}

Leave the state unchanged and return a prefix of the atom list (B.3):

\begin{itemize}
\item if \(A=\varnothing\), return \(h_{\{a\}}\);
\item otherwise, let \(j\) be the rightmost atom containing a point of \(A\), and return the indicator of
\end{itemize}
  \[
  H=\{a\}\cup\{\text{all points in atoms up to and including }j\}.
  \tag{B.5}
  \]

This contains every positive sample point. It contains no negative sample point: a negative point in an earlier atom would give a forced inversion, while a negative point in the same atom as a positive point could only occur inside \(F\), which is excluded in this case.

The set \(H\) contains either all or none of \(F\). Consequently it is a nonempty prefix of \textbf{every} order in \(\mathcal E(K)\).

Finally, store the new answer as \(\mathsf{Ans}[S]\).

\textbf{Exhaustiveness.} After contradictions are removed, two distinct sample points are incomparable precisely when both lie in \(F\). Thus, if neither (i) nor (ii) applies, every positive--negative pair is comparable and correctly ordered, or one side is empty. This is exactly case (iii). \qed

The claim is about preserving a completely free residual block---not about the number of bits in a returned concept. A reply may encode many comparisons concerning already committed points.

\medskip

\subsection*{3. Prediction rule and the full invariant}

The following is an \textbf{offline transcript-construction device}. Its conversion into a legal oblivious instance, including compliance with the label-before-prediction convention of the transductive protocol (Section 2), is proved in Lemma B.5.

When the deterministic learner produces prediction \(\widehat y\) on the next instance \(x\):

\begin{itemize}
\item if \(x=a\) or \(x\in L\), set \(y=1\), with no state change;
\item if \(x\in R\), set \(y=0\), with no state change;
\item if \(x\in F\), set
\end{itemize}
  \[
  y=1-\widehat y.
  \]
  Remove \(x\) from \(F\), and:
  \[
  \begin{cases}
  L\leftarrow L\cdot x,&y=1,\\
  R\leftarrow x\cdot R,&y=0.
  \end{cases}
  \tag{B.6}
  \]

\subsubsection*{Lemma B.3 --- Strong invariant, including all recorded answers}

Starting from \(L=R=\varnothing\), \(F=V\), the construction maintains all of the following after every event.

\begin{enumerate}
\item \textbf{Ordered partition and revealed labels.}   The lists \(L,R\) and set \(F\) partition \(V\). No point of \(F\) has arrived. Every revealed label is \(1\) on \(L\cup\{a\}\) and \(0\) on \(R\).
\end{enumerate}

\begin{enumerate}
\item \textbf{Universal validity of affirmative records.}   For every recorded affirmative answer \(h_H\), the same complete vector \(h_H\) is a concept in \(C_\tau\) consistent with its query, for every \(\tau\in\mathcal E(L,F,R)\).
\end{enumerate}

\begin{enumerate}
\item \textbf{Universal validity of negative records.}   For every query recorded as \(\bot\), that query is unrealizable by \(C_\tau\), for every \(\tau\in\mathcal E(L,F,R)\).
\end{enumerate}

\begin{enumerate}
\item \textbf{An entire interval of feasible targets.}   For every permutation \(\pi(F)\) and every \(k\in\{0,\ldots,|F|\}\), the prefix with positive set
\end{enumerate}
   \[
   \{a\}\cup L\cup\{\pi_1,\ldots,\pi_k\}
   \tag{B.7}
   \]
   agrees with every revealed label.

In particular, a consistent order and target always exist.

\textbf{Proof.} Initially there are no records or revealed labels, and every prefix in (B.7) is a valid nonempty prefix.

Every state-changing operation refines the compatible-order set:

\begin{itemize}
\item a query pin or a free prediction labeled \(1\) restricts the old free permutation to begin with the extracted point;
\item a free prediction labeled \(0\) restricts it to end with that point.
\end{itemize}

Thus previously valid affirmative and negative records remain valid in every new compatible order. New query records are valid by Lemma B.2.

For completeness, target feasibility also persists as a joint order--target statement, not merely as an order statement. If a point \(u\) is appended to \(L\), a new free permutation \(\pi'\) and cut \(k\) correspond to the old free permutation
\[
(u,\pi')
\]
and old cut \(k+1\). If \(u\) is prepended to \(R\), they correspond to the old permutation
\[
(\pi',u)
\]
and old cut \(k\). Hence every new pair in (B.7) was an admissible old pair. An extracted free point has not previously received a label, so the new commitment does not contradict history.

For prediction events, all new targets give the just-revealed label required by (B.6). An arriving point already in \(L\), \(R\), or the anchor has its committed label, so no change is needed.

A repeated query changes neither state nor table. Its stored answer is legal by items 2 or 3. This also proves the cache branch of Lemma B.2.

Finally, \(\mathcal E(L,F,R)\) is nonempty because every finite set \(F\), including the empty set, has a permutation; (B.7) always contains the anchor and therefore defines a valid concept. \qed

\textbf{A useful strengthening.} Every binary labeling of the current \(F\) remains feasible: order its desired \(1\)-points before its desired \(0\)-points and choose the corresponding cut in (B.7). Thus no hidden target-label restriction on the residual free points has been silently accumulated.

\medskip

\subsection*{4. Accounting}

\subsubsection*{Lemma B.4 --- Query pins pay for every possibly avoided free-point mistake}

Under either convention, let \(p\) be the number of points removed from \(F\) by queries, and let \(f\) be the number removed when they arrive. After all \(T\) predictions,
\[
n=p+f,\qquad p\le Q,\qquad f\le M.
\tag{B.8}
\]
Consequently,
\[
M+Q\ge n=T-\varepsilon.
\tag{B.9}
\]

\textbf{Proof.} The potential
\[
\Phi=|F|
\]
starts at \(n\) and ends at \(0\), because every point of \(V\) appears exactly once in the transductive protocol.

A query decreases \(\Phi\) by at most one. Cached queries, contradictory queries, existing-inversion queries, and affirmative replies do not decrease it. Thus \(p\le Q\).

A prediction decreases \(\Phi\) only if its own instance is still free. Such a prediction is wrong by (B.6). Hence \(f\le M\).

These are the only ways points leave \(F\), proving (B.8) and (B.9). The anchor's possible mistake in the nonempty-prefix convention is simply an additional, uncharged mistake. \qed

\medskip

\subsection*{5. Freezing into an oblivious instance}

\subsubsection*{Lemma B.5 --- Memoryless freezing and exact deterministic replay}

For either convention, every finite constructed transcript prefix can be frozen into a fixed order, target, and globally legal memoryless ERM function reproducing that prefix.

For a completed \(T\)-round construction, write the final state as \((L,\varnothing,R)\). An explicit freezing is
\[
\tau^*=a\cdot L\cdot R,
\qquad
z^*=\operatorname{last}(a\cdot L).
\tag{B.10}
\]
The fixed target therefore has positive set \(\{a\}\cup L\).

Let \(\tau^*=(u_1,\ldots,u_{n+1})\), and let \(H_j=\{u_1,\ldots,u_j\}\). Define
\[
O^*(S)=
\begin{cases}
\mathsf{Ans}[S],
   &S\in\operatorname{dom}(\mathsf{Ans}),\\[2mm]
h_{H_{j(S)}},
   &S\notin\operatorname{dom}(\mathsf{Ans})
       \text{ and a consistent prefix exists},\\[2mm]
\bot,
   &\text{otherwise},
\end{cases}
\tag{B.11}
\]
where
\[
j(S)=\min\{j\in\{1,\ldots,n+1\}:h_{H_j}\text{ is consistent with }S\}.
\tag{B.12}
\]

The table in (B.11) is an immutable parameter of the function, not mutable memory of calls during the actual interaction.

\textbf{Proof.}

\subsubsection*{Legality on every sample}

By Lemma B.3, every recorded affirmative vector is a prefix concept of the final order and is consistent with its recorded sample. Every recorded \(\bot\) sample is genuinely unrealizable by that same class.

For an unrecorded sample, (B.12) searches exactly all concepts of \(C_{\tau^*}\). It returns a consistent one precisely when one exists. This includes:

\begin{itemize}
\item \(S=\varnothing\), for which the search succeeds;
\item contradictory samples, for which it fails;
\item samples containing a negative label on the anchor, for which it fails.
\end{itemize}

Thus (B.11) is a legal consistency ERM oracle on its whole domain. Its value depends only on \(S\).

The returned objects are the stipulated complete label vectors. A vector recorded earlier may have had different possible endpoint names in different unfinished orders; after freezing it is exactly \(c_{\max_{\tau^*}H}\). No endpoint identity was additionally promised by the interface.

\subsubsection*{Target legality}

By Lemma B.3, the target in (B.10) agrees with every revealed label. It is a genuine concept, including when \(L=\varnothing\): then \(z^*=a\).

\subsubsection*{Replay}

Run the same deterministic learner on the fixed domain, announced sequence, target \(c_{z^*}\), and oracle \(O^*\).

Induct on successive interaction events. The initial input is identical. If the histories agree so far, determinism implies that the learner's next operation, query set, or prediction agrees with the constructed one.

\begin{itemize}
\item At a query event, (B.11) returns the recorded answer. This includes every repetition of a previously queried set.
\item At a label-revelation event, the fixed target returns exactly the constructed label.
\end{itemize}

The histories therefore remain identical. Complete returned vectors agree on every point of \(X\), so arbitrary evaluations of returned concepts also agree.

For a finite, nonterminal transcript prefix, choose any \(\tau^*\in\mathcal E(L,F,R)\) and any target cut from (B.7), then use the same definition (B.11). The same argument reproduces that prefix.

All adaptive choices have thus been compiled into one fixed triple. In its actual execution, each label \(c_{z^*}(x_t)\) is fixed before the learner predicts, as required by the transductive protocol. \qed

\textbf{Finiteness detail.} It is enough to construct a transcript until all \(T\) predictions occur or until \(n\) oracle calls have occurred. In the latter case, freeze the current finite prefix using Lemma B.5; its replay already has \(Q\ge n\). Otherwise the completed construction is finite. Thus an infinite-query behavior cannot obstruct the existential construction.

\subsubsection*{Completion of Theorem~\ref{thm:A} and sharpness of the constants}

Lemma B.4 gives the additive bound for the constructed transcript, and Lemma B.5 transfers it to one fixed oblivious instance.

The constants cannot be improved uniformly over \(T\ge1\):

\begin{itemize}
\item Under the main convention, at \(T=1\) a zero-query fixed-label predictor always has \(M+Q\le1\). Thus no smaller universal additive constant is possible.
\item Under the nonempty-prefix convention, at \(T=1\) the only target labels the sole instance by \(1\). Predicting \(1\) without querying gives \(M+Q=0\), ruling out a universal additive constant smaller than \(1\).
\end{itemize}

These witnesses certify only the additive constants; they are not an asserted general error--query upper tradeoff. \qed

\medskip

\subsection*{6. Separation and relation to AHR25}

\subsubsection*{Corollary --- Deterministic low-mistake learning needs nearly \(T\) queries}

Under the empty-prefix convention, suppose a deterministic learner guarantees at most \(m(T)\) mistakes uniformly over the allowed fixed instances and legal memoryless ERM oracles. Then
\[
Q_{\mathrm{worst}}(\mathcal A,T)\ge T-m(T).
\tag{B.14}
\]
Under the nonempty-prefix convention, the corresponding bound is
\[
Q_{\mathrm{worst}}(\mathcal A,T)\ge T-1-m(T).
\tag{B.15}
\]

\textbf{Proof.} Apply Theorem~\ref{thm:A}. On the instance it provides, \(M\le m(T)\), so \(Q\ge T-\varepsilon-m(T)\). \qed

In particular, a deterministic \(O(\log T)\)-mistake guarantee requires
\[
Q_{\mathrm{worst}}\ge T-\varepsilon-O(\log T).
\tag{B.16}
\]

By the randomized theorem of AHR25 (Thm.~4.5(4)), both expected mistakes and expected calls can instead be \(O(\log T)\). Consequently:

\begin{itemize}
\item randomized expected total cost \(M+Q\) is \(O(\log T)\);
\item every deterministic learner has worst-case total cost at least \(T-\varepsilon\).
\end{itemize}

Equivalently, using the dimension conventions of Section 2, the gap is exponential when parameterized by \(d=\Theta(\log T)\): logarithmic-in-\(T\) cost versus linear-in-\(T\) cost, or \(O(d+1)\) versus \(2^{\Omega(d)}\).

This establishes a positive answer to the open question of AHR25 about the power of randomization \textbf{within the threshold-family, consistency-ERM, oblivious-instance setting of this paper}.

\subsubsection*{Why Theorem 4.4 of AHR25 is not the proof of this result}

The statement of Theorem 4.4 of AHR25 is:

\begin{quote}``Let \(\mathcal F\) be the family of all classes with Littlestone dimension \(d_{\mathrm{LD}}\). For \(T\ge 2^{d_{\mathrm{LD}}+1}\), when having access to the ERM oracle, if fewer than \(T/2\) queries are made, at least \(2^{d_{\mathrm{LD}}}-1\) mistakes will be made.''\end{quote}

Its proof contains the sentence:

\begin{quote}``An essential property is that each query provides information about at most one point.''\end{quote}

The present proof implements a related one-point charging principle, but that theorem is \textbf{not invoked as a lower-bound lemma}. Its construction uses threshold classes on equivalence classes, rather than the present threshold classes on distinct totally ordered points. Moreover, for either dimension convention of Section 2,
\[
d=\lfloor\log_2(T+1)\rfloor
\quad\text{or}\quad
d=\lfloor\log_2T\rfloor,
\]
the required condition \(T\ge2^{d+1}\) fails.

Here, forced-realizable queries are explicitly answered by proper full concepts, and all replies are then frozen into one memoryless oracle.

\subsubsection*{Why the construction does not contradict Theorem 4.5(4) of AHR25}

For a randomized learner, performing the construction separately for each random tape would generally produce
\[
(\tau_\omega,z_\omega,O_\omega)
\]
depending on that tape. In particular, which negative sample point is pinned may depend on the learner's random query.

This gives an instance \textbf{after seeing randomness}, not one fixed oblivious instance. The invalid quantifier exchange would be
\[
\forall\omega\ \exists I_\omega
\quad\Longrightarrow\quad
\exists I\ \text{with the same expected lower bound over }\omega.
\]
Determinism permits the replay argument for one learner execution; it does not justify this exchange.

\medskip

\section{Proof of Theorem~\ref{thm:C}}
\label{sec:proofC}
\subsection*{1. Theorem~\ref{thm:C}}

All logarithms in the proof are natural.

Let
\[
\epsilon=
\begin{cases}
0,&\text{under convention (E)},\\
1,&\text{under convention (N)},
\end{cases}
\qquad n=T-\epsilon.
\]

Under (E), there are \(n=T\) prediction points and an additional sentinel \(a=s\). Under (N), there are \(n=T-1\) hard prediction points and one additional prediction point \(a\). The point \(a\) will be the minimum of the total order and will always have target label \(1\).

\subsubsection*{Theorem~\ref{thm:C} --- expected-query and worst-case-query lower bounds}

For every \(T\ge1\), under either convention above, there exist:

\begin{enumerate}
\item a finite distribution \(\mathcal D_T^\epsilon\) over total orders and target thresholds;
\item a fixed transductive sequence of \(T\) distinct points; and
\item an explicit deterministic, memoryless consistency-type ERM oracle
\end{enumerate}
   \[
   O(S,\preceq),
   \]
   depending only on the sample and the total order,

such that every randomized learner satisfies
\[
\boxed{
\mathbb E[M]\ \ge\
\frac{(T+1-\epsilon)\,128^{-\mathbb E[N_{\rm qry}]}-1}{2}.
}
\tag{C.1}
\]
Here both expectations are over \(\mathcal D_T^\epsilon\) and the learner's private randomness, and the assertion is substantive when the expected number of queries is finite.

Consequently:

\begin{itemize}
\item \textbf{Expected-query budget.} If \(\mathbb E[N_{\rm qry}]\le Q\), then
\end{itemize}
  \[
  \mathbb E[M]\ge
  \max\left\{0,\frac{(T+1-\epsilon)128^{-Q}-1}{2}\right\}.
  \tag{C.2}
  \]
\begin{itemize}
\item \textbf{Worst-case-query budget.} If \(N_{\rm qry}\le Q\) almost surely, the same bound holds.
\end{itemize}

In particular, for both conventions,
\[
\boxed{
\mathbb E[M]\ge \frac12\,T\left(\frac1{128}\right)^Q-\frac12.
}
\tag{C.3}
\]

The expected-query hypothesis in (C.2) may be imposed merely on the average over the displayed hard distribution. In particular, the uniform, per-instance expected-query guarantee in the cost accounting of Section~2 implies it.

The construction and proof below are self-contained.

\medskip

\subsection*{2. The hard distribution and the actual oracle}

\subsubsection*{2.1 A balanced interval tree}

Identify the hard points with \(x_1,\ldots,x_n\), in their arrival order.

Build a binary tree on the interval \([n]\). An interval of size \(m\ge2\) is split into its first \(\lfloor m/2\rfloor\) indices and its remaining \(\lceil m/2\rceil\) indices. Singletons are leaves.

For \(n\ge1\), draw
\[
J\sim\operatorname{Unif}([n]),
\qquad
Y_1,\ldots,Y_n\stackrel{\rm iid}{\sim}\operatorname{Bernoulli}(1/2),
\]
independently.

Define a priority permutation \(\pi_J\) as follows:

\begin{itemize}
\item traverse the path from the root to leaf \(J\);
\item at every internal node, append all indices in the sibling not containing \(J\), in increasing numerical order;
\item finally append \(J\).
\end{itemize}

Thus, every sibling encountered earlier has higher, or ``outer,'' priority than every point in the current interval containing \(J\).

\subsubsection*{2.2 From priorities and labels to a genuine total order}

The total order is
\[
a
\ \prec\
\bigl(x_i:Y_i=1,\text{ in }\pi_J\text{ order}\bigr)
\ \prec\
\bigl(x_i:Y_i=0,\text{ in reverse }\pi_J\text{ order}\bigr).
\tag{C.4}
\]

The target threshold ends at the last \(Y_i=1\) point in the first list, or at \(a\) if that list is empty. Therefore
\[
c_z(a)=1,\qquad c_z(x_i)=Y_i.
\tag{C.5}
\]

This is a total-order threshold class, not a threshold class on equivalence classes: every point occupies its own distinct position, and every point-prefix belongs to the class.

The transductive sequence is:

\begin{itemize}
\item under (E): \(x_1,\ldots,x_n\), with \(a=s\) a non-instance sentinel;
\item under (N): \(x_1,\ldots,x_n,a\).
\end{itemize}

If \(n=0\), which occurs only for \(T=1\) under (N), use the singleton order consisting of \(a\), with target label \(1\).

All random choices are made before interaction. This defines \(\mathcal D_T^\epsilon\), independently of the learner.

\subsubsection*{2.3 The actual, memoryless oracle}

Use the \textbf{minimum consistent prefix oracle}. For completeness, the following definition is valid on any nonempty finite totally ordered domain.

Given \(S\subseteq X\times\{0,1\}\):

\begin{enumerate}
\item If \(S\) assigns both labels to some point, return \(\bot\).
\item Let
\end{enumerate}
   \[
   P=\{x:(x,1)\in S\},\qquad B=\{x:(x,0)\in S\}.
   \]
\begin{enumerate}
\item Set
\end{enumerate}
   \[
   p=
   \begin{cases}
   \max_\preceq P,&P\ne\varnothing,\\
   \min_\preceq X,&P=\varnothing.
   \end{cases}
   \]
\begin{enumerate}
\item If some \(b\in B\) satisfies \(b\preceq p\), return \(\bot\). Otherwise return the full vector \(c_p\).
\end{enumerate}

This rule depends only on \((S,\preceq)\), not on \(J\), the target, the history, or the learner's private randomness.

It is a legal oracle: every threshold containing \(P\) must contain the prefix ending at \(p\); thus an element of \(B\) in that prefix makes consistency impossible. Otherwise \(c_p\) itself is consistent. This also covers empty and one-sided samples. On our support, every concept contains \(a\), so a query containing \((a,0)\) is correctly rejected.

\medskip

\subsection*{3. Removing failure answers by a stronger oracle}

The next reduction is important: it handles arbitrary \(B\), including all information carried by \(\bot\) answers.

\subsubsection*{Lemma C.1 --- prefix-max domination}

On the constructed instances, every call to \(O\) can be simulated using one call to a stronger oracle which, on any set \(A\subseteq[n]\), returns the full prefix ending at
\[
\max_\preceq\{x_i:i\in A\},
\]
or at \(a\) when \(A=\varnothing\).

\paragraph{Proof}

Take \(A\) to be the hard points assigned label \(1\) by the sample. The stronger response is precisely the candidate minimum prefix in the definition of \(O\). Checking the sample's \(0\)-labeled points against that full vector determines whether \(O\) returns this prefix or \(\bot\). Contradictory samples and \((a,0)\) may be rejected immediately.

Thus one stronger call determines the original answer for every sample, regardless of its size or the arrival status of its points. \qed

We now strengthen this prefix-max oracle further.

\subsubsection*{The auxiliary revelation oracle}

The auxiliary oracle maintains an interval \(I\) containing \(J\), initially \(I=[n]\). It reveals to the learner:

\begin{itemize}
\item every target label outside \(I\);
\item the priority order of all those outside points;
\item the interval \(I\) itself.
\end{itemize}

Such information may include future labels. This is deliberately additional information.

For a query \(A\subseteq[n]\), it proceeds as follows.

\begin{enumerate}
\item If \(A\setminus I\) contains a point with target label \(0\), stop.
\item If \(A\cap I=\varnothing\), stop.
\item If \(I\) is a singleton, reveal its label, set \(I=\varnothing\), and stop.
\item Otherwise reveal which child of \(I\) contains \(J\), reveal all target labels in the other child, replace \(I\) by the child containing \(J\), and repeat.
\end{enumerate}

At termination it returns the requested full prefix as well as the auxiliary information.

The entire procedure costs \textbf{one query}, however many internal revelations occur.

This is only an auxiliary, stronger information source. It is not asserted to be the actual ERM oracle. Its memory and its access to target labels do not enter the definition of \(O(S,\preceq)\).

\subsubsection*{Lemma C.2 --- the auxiliary oracle simulates the full prefix}

At every stopping point above, the requested prefix is determined by the revealed information. In particular, an original learner can be simulated by a learner using the auxiliary oracle, with the same number of queries and the same prediction distribution.

\paragraph{Proof}

All points outside the current \(I\) precede all points inside \(I\) in priority order.

Suppose \(A\setminus I\) contains a target-negative point. The maximum of \(A\) in the total order is then the \textbf{outermost target-negative point of \(A\)}: the one appearing earliest in \(\pi_J\). It lies outside \(I\), and its identity is known.

Its prefix contains:

\begin{itemize}
\item every point of \(I\);
\item every outside target-positive point; and
\item exactly those outside target-negative points whose priority is no earlier than its own.
\end{itemize}

The entire vector is therefore known.

If there is no outside target-negative point and \(A\cap I=\varnothing\), all points of \(A\) are revealed target-positive points. Their total-order maximum is their latest point in priority order. Its prefix contains no point of \(I\), and its values outside \(I\) are known. The empty-\(A\) case returns the prefix \(\{a\}\).

The singleton revelation makes all remaining information available. The procedure terminates because every nonterminal refinement reduces the interval size.

Combining this computation with Lemma C.1 reconstructs every original response, including \(\bot\). An original learner may simply ignore the additional information. \qed

\medskip

\subsection*{4. The posterior invariant}

\subsubsection*{Lemma C.3 --- conditional independence inside the active interval}

At every point of the auxiliary interaction, including between internal revelations of a query, conditional on the complete auxiliary history:

\begin{enumerate}
\item \(J\) is uniform in the current nonempty interval \(I\);
\item all labels of unarrived points in \(I\) are independent fair bits;
\item these unarrived labels are independent of \(J\).
\end{enumerate}

The statement remains valid for randomized learners.

\paragraph{Proof}

Initially this is the definition of the distribution.

Suppose the statement holds for \(I\). The next child containing \(J\) has probability equal to its size divided by \(|I|\). Conditional on that child, \(J\) is uniform within it.

The labels revealed in the sibling are independent of \(J\) within the retained child and independent of all unarrived labels in that child. The priority ordering of the revealed sibling is fixed by its numerical indices and does not depend on the remaining location of \(J\).

The decision to stop or continue uses only:

\begin{itemize}
\item the chosen query;
\item the retained child's identity;
\item the revealed labels outside it; and
\item whether the query has any point in the retained child.
\end{itemize}

It does not depend on the location of \(J\) within that child or on its unrevealed future labels. Hence conditioning on that decision preserves the asserted law. The final full-prefix response is determined by the revealed information, by Lemma C.2, and creates no additional conditioning.

An arriving label in \(I\) is a fair bit independent of \(J\) and the remaining unarrived labels. Revealing it therefore also preserves the invariant.

For randomized learners, include the learner's private seed in the conditioning, or apply the same induction after fixing that seed. Its independence from the hard instance supplies the base case. \qed

This explicitly replaces, rather than assumes, exchangeability after failure answers.

\medskip

\subsection*{5. Charging free labels to predictions}

Call a prediction \textbf{unrevealed} if its point is still in the active interval immediately before that prediction. Let \(H\) be the number of such hard-point predictions made so far; \(H\) is a cumulative count over the whole interaction and never decreases when the active interval shrinks (it must not be confused with the number \(k\) of already-predicted points inside the current interval, for which \(H\ge k\)).

Queries between a prediction and its label feedback can be treated as post-feedback queries in a stronger protocol: provide the label immediately after the prediction, and let a simulator ignore it until the original learner would have received it. The prediction loss is unchanged. Thus it suffices to analyze queries between prediction-feedback steps.

\subsubsection*{Lemma C.4 --- mistake charge and the potential}

Let the active interval have size \(m\). Let \(k\) of its points have already been predicted, and let \(r=m-k\). Then:

\begin{enumerate}
\item \(H\ge k\);
\item \(\mathbb E[M]\ge \frac12\mathbb E[H_{\rm final}]\);
\item the potential
\end{enumerate}
   \[
   \Phi=\ln(H+r+1)
   \tag{C.6}
   \]
   is unchanged by prediction-feedback steps;
\begin{enumerate}
\item one binary interval refinement decreases \(\Phi\) by at most \(\ln2\).
\end{enumerate}

\paragraph{Proof}

The active intervals are nested. Every already-predicted point still belonging to the current interval also belonged to the active interval when it was predicted. Hence \(H\ge k\).

By Lemma C.3, an unrevealed prediction faces an independent fair target bit. Its conditional expected loss is exactly \(1/2\), regardless of the prediction distribution. Summing these conditional expectations proves assertion 2.

On an unrevealed prediction, \(H\) increases by one and \(r\) decreases by one. Otherwise neither changes. This proves assertion 3.

For assertion 4, write \(m',k'\) for the retained child's size and number of already-predicted points. We have
\[
m'\ge\lfloor m/2\rfloor,\qquad k'\le k,\qquad H\ge k.
\]
Consequently,
\[
H+m-k+1
\le 2(H+m'-k'+1).
\tag{C.7}
\]
Indeed, with \(d=H-k\ge0\), the left side is \(d+m+1\), whereas the new quantity is at least \(d+m'+1\), and
\[
d+m+1\le2(d+m'+1).
\]
Taking logarithms proves the assertion. \qed

The \(+1\) is essential for handling empty residual sets and the last leaf without undefined logarithms.

\medskip

\subsection*{6. The single-query information lemma}

\subsubsection*{Lemma C.5 --- bounded expected logarithmic decrease per query}

Conditional on any auxiliary history immediately before a query,
\[
\boxed{
\mathbb E[\Phi_{\rm before}-\Phi_{\rm after}\mid\text{history}]
\le 7\ln2.
}
\tag{C.8}
\]

This holds for every query set \(A\), including sets chosen using arbitrarily many previously observed labels.

\paragraph{Proof}

During one query, \(H\) is fixed.

Call a state \textbf{early} if
\[
k<\lfloor m/2\rfloor,
\]
and \textbf{late} otherwise.

\subsubsection*{Late states}

In a late state,
\[
r=m-k\le k+1\le H+1.
\]
Therefore
\[
H+r+1\le2(H+1).
\tag{C.9}
\]
The potential at the end of the query is at least \(\ln(H+1)\). Thus, from the first late state onward, the entire remaining potential decrease is at most \(\ln2\), regardless of how many further revelations occur.

This also covers a singleton interval.

\subsubsection*{Early states}

Consider an early state in which the query needs another refinement. Let \(I_L,I_R\) be the left and right children.

Because arrival order is numerical order and fewer than \(|I_L|\) points of \(I\) have arrived, \textbf{every point of \(I_R\) is unarrived}. Its labels are therefore independent fair bits by Lemma C.3.

There are two exhaustive cases.

\begin{itemize}
\item If \(A\cap I_R=\varnothing\), then when \(J\in I_R\), the new active interval contains no point of \(A\), so the query stops. This event has probability
\end{itemize}
  \[
  \frac{|I_R|}{m}\ge\frac12.
  \]

\begin{itemize}
\item If \(A\cap I_R\ne\varnothing\), consider the event that \(J\in I_L\) and at least one point of \(A\cap I_R\) has target label \(0\). The right child is then revealed, and the query stops because an outside target-negative point of \(A\) has been found. Its conditional probability is
\end{itemize}
  \[
  \frac{|I_L|}{m}
  \left(1-2^{-|A\cap I_R|}\right)
  \ge \frac13\cdot\frac12
  =\frac16.
  \tag{C.10}
  \]

Thus, after each early-state refinement, the conditional probability of continuing to another early-state refinement is at most \(5/6\).

Let \(D\) be the number of early-state refinements before the query stops or first reaches a late state. No independence between refinements is required:
\[
\Pr(D\ge j)\le(5/6)^{j-1},
\qquad
\mathbb E[D]\le6.
\tag{C.11}
\]
Each refinement costs at most \(\ln2\) in potential, by Lemma C.4. The remaining late-state part costs at most another \(\ln2\). Hence
\[
\mathbb E[\Phi_{\rm before}-\Phi_{\rm after}]
\le6\ln2+\ln2.
\]
This proves (C.8). \qed

The two mechanisms in this lemma are complementary:

\begin{itemize}
\item querying genuinely future points encounters an independent wrong-side point with constant probability;
\item exploiting many already known labels is permitted, but their predictions have already contributed to \(H\).
\end{itemize}

\medskip

\subsection*{7. Proof of Theorem~\ref{thm:C}}

Initially,
\[
H=0,\qquad r=n,\qquad \Phi_{\rm initial}=\ln(n+1).
\]
After all hard points have been predicted,
\[
r=0,\qquad \Phi_{\rm final}=\ln(H_{\rm final}+1).
\]

Prediction-feedback steps leave the potential unchanged. Let \(\Delta_j\ge0\) be its decrease during the \(j\)-th query, setting \(\Delta_j=0\) when that query does not occur.

Lemma C.5 gives
\[
\mathbb E[\Delta_j]\le 7\ln2\,\Pr(N_{\rm qry}\ge j).
\]
Summing nonnegative quantities,
\[
\begin{aligned}
\ln(n+1)-\mathbb E[\ln(H_{\rm final}+1)]
&=\mathbb E\!\left[\sum_j\Delta_j\right]\\
&\le7\ln2\sum_j\Pr(N_{\rm qry}\ge j)\\
&=7\ln2\,\mathbb E[N_{\rm qry}].
\end{aligned}
\tag{C.12}
\]
This argument directly permits an adaptive, random number of queries. It does not condition the hard distribution on a low-query event.

By concavity of \(\ln\),
\[
\ln(\mathbb E[H_{\rm final}]+1)
\ge \mathbb E[\ln(H_{\rm final}+1)].
\]
Combining this with (C.12),
\[
\mathbb E[H_{\rm final}]
\ge(n+1)\exp\!\bigl(-7\ln2\,\mathbb E[N_{\rm qry}]\bigr)-1.
\tag{C.13}
\]
Lemma C.4 now yields
\[
\mathbb E[M]\ge
\frac{(n+1)128^{-\mathbb E[N_{\rm qry}]}-1}{2}.
\]

This lower bound holds even for learners receiving the auxiliary information. By Lemmas C.1 and C.2, it therefore holds for every learner using the actual memoryless oracle.

Substituting \(n=T-\epsilon\) proves (C.1). Monotonicity in the query budget proves both versions of (C.2).

For \(n=0\), the instance has one known positive point and the claimed lower bound is nonpositive; it is immediate. Thus all boundary cases are included. \qed

\subsubsection*{Oblivious fixed-instance consequence}

Suppose a learner satisfies the uniform expected-query guarantee of the cost accounting of Section~2:
\[
\sup_{\preceq,z}\mathbb E_{\mathcal A}[N_{\rm qry}]\le Q.
\]
Its average query count under \(\mathcal D_T^\epsilon\) is at most \(Q\). Since the distribution has finite support, some fixed support instance satisfies
\[
\mathbb E_{\mathcal A}[M]
\ge\frac{(T+1-\epsilon)128^{-Q}-1}{2}.
\]
The query guarantee also holds on that instance by hypothesis.

The order, target, sequence, and oracle are already fixed before the learner's private randomness. No adaptive-adversary solidification step is needed.

\medskip

\subsection*{8. Consequence for logarithmic query complexity}

\subsubsection*{Corollary C.9}

Under either convention, suppose a randomized learner has worst-case expected mistakes at most \(m_T\) and worst-case expected queries at most \(q_T\), in the sense of the cost accounting of Section~2. Then
\[
q_T\ge
\frac{\ln(T+1-\epsilon)-\ln(2m_T+1)}{7\ln2}.
\tag{C.14}
\]

In particular, if \(m_T=O((\ln T)^d)\) for a fixed \(d\), then
\[
q_T\ge
\frac{\ln T-d\ln\ln T-O(1)}{7\ln2}
=\Omega(\ln T).
\tag{C.15}
\]

\paragraph{Proof}

Apply Theorem~\ref{thm:C} to the learner and rearrange. \qed

Together with AHR25, Thm.~4.5(4), this proves that \(O(\log T)\) expected oracle calls are optimal in order for achieving polylogarithmic expected mistakes.

Together with Theorem~\ref{thm:A} and AHR25, Thm.~4.5(3), it also closes the stated comparison at \(O(\log T)\) mistakes:

\begin{itemize}
\item randomized query complexity: \(\Theta(\log T)\), in expectation;
\item deterministic query complexity: \(\Theta(T)\) against worst-case legal oracles and under the extremal rules of Theorem~\ref{thm:Aprime} (not under every legal pre-declared rule, Theorem~\ref{thm:rule}).
\end{itemize}

This does \textbf{not} claim sharp constants or a complete pointwise tight characterization for every possible mistake/query budget.

\subsubsection*{Relation to the one-point property}

The ``one point per query'' property used in the proof of Theorem~4.4 of AHR25 is not used here and generally fails for the present full-vector ERM interface: a query in our analysis may reveal a whole sibling interval, containing a linear number of labels. The replacement property is Lemma C.5: one query decreases a \textbf{paid logarithmic potential} by only a constant in expectation. Theorem~4.4 of AHR25 is not invoked here; its statement, and why its hypothesis fails for the threshold family, are discussed once in Appendix~\ref{sec:proofA}.

\medskip

\subsection*{9. Exact zero-query analysis for the uniform distribution}

The following results use the uniform distribution under convention (E), not the different distribution used in Theorem~\ref{thm:C}.

\subsubsection*{Proposition C.6 --- exact zero-query optimum}

Under the uniform distribution, the minimum expected number of mistakes of any zero-query learner, including randomized learners, is
\[
\boxed{
F(T)=\sum_{t=1}^{T}
\frac{\left\lfloor (t+1)^2/4\right\rfloor}{t(t+1)}.
}
\tag{C.16}
\]
In particular,
\[
F(T)\ge T/4,
\qquad
F(T)=T/4+\tfrac14\ln T+O(1).
\tag{C.17}
\]

After observing \(k\) positive and \(m\) negative labels, the next label has posterior probability
\[
\Pr(Y_{\rm next}=1\mid\text{history})
=\frac{k+1}{k+m+2}.
\tag{C.18}
\]

\paragraph{Proof}

An equivalent representation of the uniform distribution is to draw independent
\[
P,U_1,\ldots,U_T\sim\operatorname{Unif}[0,1]
\]
and set \(Y_i=\mathbf1[U_i\le P]\). Ordering the \(U_i\)'s gives a uniform permutation, and the position of \(P\) among them is uniform on \(\{0,\ldots,T\}\). The actual threshold point remains an element of \(X\): it is the last instance below \(P\), or the sentinel.

For a history with \(k\) ones and \(m\) zeros, the posterior density of \(P\) is proportional to \(p^k(1-p)^m\). Taking the ratio of the elementary integrals for its first moment proves (C.18).

After \(r\) observations, the number \(K_r\) of ones is uniform on \(\{0,\ldots,r\}\). Therefore the expected Bayes error on round \(r+1\) is
\[
\frac1{(r+1)(r+2)}
\sum_{k=0}^r\min(k+1,r-k+1)
=
\frac{\lfloor(r+2)^2/4\rfloor}{(r+1)(r+2)}.
\]
Summing proves (C.16).

A randomized prediction cannot improve on the posterior-majority decision, since its conditional loss is affine in the probability of predicting \(1\). Posterior majority attains every summand, proving exact optimality.

Finally, the round-\(t\) summand equals
\[
\frac14+
\begin{cases}
\frac1{4t},&t\text{ odd},\\[2mm]
\frac1{4(t+1)},&t\text{ even}.
\end{cases}
\]
This proves (C.17). \qed

\textbf{Convention (N) correction.} If the literal no-empty-prefix counterpart draws a uniform permutation and a uniform cut in \(\{1,\ldots,T\}\), its label distribution is the (E) distribution conditioned on a nonempty positive set. The total variation distance is \(1/(T+1)\). Since total loss lies in \([0,T]\), its optimal zero-query value \(F_N(T)\) satisfies
\[
|F_N(T)-F(T)|\le \frac{T}{T+1}<1.
\tag{C.19}
\]
In particular, \(F_N(T)\ge T/4-1\).

Notice that the exact correction in (C.17) is logarithmic and positive; the exact optimum is not \(T/4+O(1)\).

\medskip

\subsection*{10. Why the uniform distribution cannot prove Theorem~\ref{thm:C}}

\subsubsection*{Proposition C.7 --- a one-query learner for the uniform distribution}

Fix \(1\le m\le T\). Predict the first \(m\) labels by posterior majority. Then query all \(m\) observed labels once and predict the returned hypothesis on all remaining points.

For the minimum-prefix rule on the uniform distribution, this algorithm has exactly
\[
\boxed{
\mathbb E[M]=F(m)+\frac{T-m}{m+2}.
}
\tag{C.20}
\]
For \textbf{every} consistent returned threshold, it satisfies
\[
\mathbb E[M]\le F(m)+\frac{2(T-m)}{m+2}.
\tag{C.21}
\]
Thus one query suffices for \(O(\sqrt T)\) expected mistakes on the uniform distribution.

\paragraph{Proof}

Use the representation in Proposition C.6. Let
\[
L=\max\{U_i:i\le m,\ Y_i=1\},
\]
with \(L=0\) if there is no observed positive point.

The minimum-prefix hypothesis makes future mistakes exactly on \((L,P]\). Conditional on \(P=p\),
\[
\mathbb E[p-L\mid P=p]
=\int_0^p(1-v)^m\,dv.
\]
Averaging over \(p\),
\[
\mathbb E[P-L]
=\frac1{m+2}.
\tag{C.22}
\]
Each future point is independent of the first \(m\) points and \(P\), proving (C.20).

For arbitrary consistent tie-breaking, let \(U\) be the smallest observed negative coordinate, with \(U=1\) if none exists. Every consistent threshold agrees with the target outside \((L,U)\), regardless of how its choice depends on the full order. By symmetry,
\[
\mathbb E[U-P]=\frac1{m+2}.
\]
Bounding mistakes by membership in \((L,U)\) proves (C.21). \qed

For example, with \(m=\lceil\sqrt T\rceil\), using \(F(m)\le m/2\), (C.21) gives the explicit bound
\[
\mathbb E[M]\le \frac52\sqrt T+\frac12.
\tag{C.23}
\]

This is a distribution-specific upper bound, not a worst-case upper bound over all orders. It does not contradict Theorem~\ref{thm:A}. It does prove that \textbf{no positive-constant \(T\rho^Q\) lower bound with \(Q=1\) can hold on the uniform distribution}, even if the oracle's consistent tie-breaking is chosen adversarially.

\medskip

\subsection*{11. A counterexample to a naive per-history contraction bound}

\subsubsection*{Proposition C.8 --- no history-uniform constant contraction bound for the uniform distribution}

Even after regularizing block sizes by \(+1\), a single-query logarithmic contraction bound on the uniform distribution cannot hold uniformly over all positive-probability histories.

\paragraph{Proof}

Choose integers \(k\ge1\), \(u\ge1\), and \(T=k+u+1\). Put \(b=x_{k+1}\).

Before prediction, query the singleton positive sample \(\{(b,1)\}\). Consider the history in which:

\begin{enumerate}
\item the answer is the all-one vector, so \(b\) is the maximum instance in the unknown order;
\item the first \(k\) arriving labels are all \(1\);
\item the label of \(b\) is \(0\).
\end{enumerate}

Under the uniform distribution this history has probability
\[
\frac1T\cdot\frac{T}{T+1}\cdot\frac1{k+1}
=\frac1{(T+1)(k+1)}>0.
\tag{C.24}
\]

Before the next query, all \(u\) unarrived points have undetermined labels. Query the \(k\) known positive points. Let \(U'\) be the set of unarrived points above their maximum.

Removing the known maximum \(b\), the conditional model is again a uniform random order and uniform cut, now conditioned on \(k\) specified positive labels. The expected coordinate of their maximum is \(k/(k+2)\). Hence
\[
\mathbb E[|U'|\mid\text{history}]
=\frac{2u}{k+2}.
\tag{C.25}
\]
All points removed from the uncertainty set are certified positive; the points in \(U'\) remain label-undetermined.

By Jensen's inequality,
\[
\begin{aligned}
\mathbb E\!\left[
\ln\frac{u+1}{|U'|+1}
\,\middle|\,\text{history}
\right]
&\ge
\ln\frac{u+1}{2u/(k+2)+1}.
\end{aligned}
\tag{C.26}
\]
Taking \(u=(k+2)^2\), the right side becomes
\[
\ln\frac{(k+2)^2+1}{2k+5},
\]
which diverges with \(k\).

All parameters and probabilities here are finite and rational. For one explicit member,
\[
k=255,\quad u=66049,\quad T=66305,
\]
the history probability is \(1/16974336\), and
\[
\mathbb E[|U'|]=514,\qquad
\frac{\mathbb E[|U'|+1]}{u+1}=\frac{103}{13210}.
\]
Thus the expected logarithmic decrease is at least
\[
\ln(13210/103)>\ln128.
\]
The parametrized family rules out every universal constant, not merely \(7\ln2\). \qed

The failure is not caused by taking a logarithm of zero. It persists with \(+1\) regularization. The posterior on the threshold within the remaining label-undetermined set is not the fresh uniform posterior assumed by a naive potential argument.

\medskip

\section{Proof of Theorem~\ref{thm:B}}
\label{sec:proofB}
\subsection*{1. Conventions and the counterexamples}

Budgets below are nonnegative integers. For a real-valued budget, replace \(Q\) by \(\lfloor Q\rfloor\).

Write \(M_E^*(T,Q)\) and \(M_N^*(T,Q)\) for the deterministic minimax mistake counts under conventions E and N, respectively. Lower bounds use the finite domains fixed in Section 2:

\begin{itemize}
\item E: \(X=\{s\}\cup V\), where \(s\) is the least point and \(V\) consists of the \(T\) instances;
\item N: \(X=V\).
\end{itemize}

All oracle answers are complete label vectors. No threshold endpoint name or additional representation information is assumed.

\subsubsection*{Theorem D.1 --- The conjectured formula is false (conventions E and N)}

In the model of Section 2,
\[
M_E^*(5,3)=3,
\qquad
M_N^*(6,4)=3.
\]

In particular, the proposed E formula
\[
M_E^*(T,Q)=\max\{T-Q,\lfloor\log_2(T+1)\rfloor\}
\]
is false. Consequently no formula that assigns the value 2 to \(M_N^*(6,4)\) can be correct either.

The proof is analytic.

\medskip

\subsection*{2. What is correct in the protection algorithm}

\subsubsection*{Lemma D.2 --- Computable knowledge and sound protection}

Maintain a partial order \(P_t\) on the instance set \(V\). Its relations comprise:

\begin{enumerate}
\item relations learned from previous oracle answers;
\item every relation \(a\prec b\) implied by an observed positive \(a\) and an observed negative \(b\);
\item their reflexive/transitive closure.
\end{enumerate}

Let \(A_t,B_t\) be the observed positive and negative sets. Define the retained version space
\[
\mathcal I_t^E
=
\{I\subseteq V:I\text{ is an ideal of }P_t,\ A_t\subseteq I,\ I\cap B_t=\varnothing\},
\]
and
\[
\mathcal I_t^N=\mathcal I_t^E\setminus\{\varnothing\}.
\]

An ideal is a downward-closed subset. Every ideal is a prefix of some linear extension of \(P_t\): first topologically sort the ideal, then its complement. Thus these version spaces describe precisely the target vectors compatible with the information deliberately retained by the algorithm.

Define
\[
D_1=\bigcap_{I\in\mathcal I_t}I,\qquad
D_0=V\setminus\bigcup_{I\in\mathcal I_t}I,\qquad
U=V\setminus(D_0\cup D_1).
\]
These sets are computable without knowing the true order. Realizability guarantees \(\mathcal I_t\ne\varnothing\).

Under E, they have the simpler reachability descriptions
\[
D_1=\{v:\exists w\in A_t,\ v\preceq_{P_t}w\},
\qquad
D_0=\{v:\exists w\in B_t,\ w\preceq_{P_t}v\}.
\]
Indeed, the smallest feasible ideal is \(\downarrow A_t\); if \(v\notin D_0\), then \(\downarrow(A_t\cup\{v\})\) avoids \(B_t\). Under N, enumeration of nonempty ideals also incorporates the nonemptiness inference that the reachability tests alone can miss.

Every previously observed point is determined. Consequently, if the current point \(x\) and another point \(u\) are both undetermined, then \(u\) has not arrived.

Query
\[
S=\{(x,1),(u,0)\}.
\]

\begin{itemize}
\item \textbf{Answer \(\bot\).} This happens exactly when \(u\prec x\). If \(x\prec u\), the threshold \(c_x\) realizes \(S\); the reverse ordering makes \(S\) impossible. Predict \(0\). If this prediction is wrong, then \(x=1\), and \(u\prec x\) forces \(u=1\).
\end{itemize}

\begin{itemize}
\item \textbf{Answer \(h\).} Put \(H=\{v\in V:h(v)=1\}\). Then \(x\in H\), \(u\notin H\), and
\end{itemize}
  \[
  H\prec V\setminus H.
  \]
  Predict \(1\). If this prediction is wrong, then \(x=0\). Every point outside \(H\) follows \(x\), so every such point is \(0\), including \(u\).

The latter inference uses the fact that \(H\) is a threshold prefix, not merely that it contains \(x\). Both branches work for every legal consistent return value. \qed

\subsubsection*{Lemma D.3 --- Correct accounting, including early termination}

Run the protection algorithm with an internal query cap \(k\): query at every undetermined arrival whenever budget remains and another undetermined point exists. Let \(q\le k\) be the number of calls actually made. Then
\[
\boxed{M\le T-q.}
\]

If \(q<k\), the valid additional conclusion is
\[
\boxed{M\le q+1,}
\]
not \(M\le1\). Therefore
\[
\boxed{M\le\max\{T-k,k\}.} \tag{D.1}
\]

\textbf{Proof.} Let \(D\) be the set of arrivals that were undetermined immediately before their prediction/query processing. Let \(b\) be the number of mistaken query rounds.

Each mistaken query protects a previously undetermined future point. Once protected, that point remains determined. Hence the protected points from the \(b\) mistaken queries are distinct and do not belong to \(D\). Thus
\[
|D|+b\le T.
\]
At most \(|D|-q\) mistakes occur on nonquery rounds, giving
\[
M\le b+|D|-q\le T-q.
\]

If \(q<k\), budget was never exhausted. Any undetermined arrival at which no query was made must therefore have been the sole remaining undetermined point. After its label is observed, no future undetermined arrival is possible. Thus at most one member of \(D\) is a nonquery round:
\[
|D|\le q+1,\qquad M\le q+1\le k.
\]
Combining this with the \(q=k\) case proves (D.1). \qed

In particular, choosing
\[
k=\min\{Q,\lfloor T/2\rfloor\}
\]
gives the valid general upper bound
\[
M\le \max\{T-Q,\lceil T/2\rceil\}. \tag{D.2}
\]

For Theorem D.1:

\begin{itemize}
\item E, \(T=5,Q=3\): take \(k=2\), obtaining \(M\le3\);
\item N, \(T=6,Q=4\): take \(k=3\), obtaining \(M\le3\).
\end{itemize}

The remaining work is to prove matching lower bounds for \textbf{every} deterministic learner.

\medskip

\subsection*{3. An analytic five-point lower bound under E}

Throughout this section there are five instances, an E sentinel, and no observed labels unless explicitly stated. A poset state means that every linear extension of that poset remains available to the adversary. Providing the poset to the learner for free only strengthens the learner.

\subsubsection*{Lemma D.4 --- Antichain cube}

Suppose no queries remain and the available orders include all linear extensions of a poset \(P\) containing an antichain \(W\) of size \(r\). Then the adversary can force \(r\) mistakes in the prescribed transductive sequence.

\textbf{Proof.} Let
\[
D=\{v\notin W:v<_P w\text{ for some }w\in W\}.
\]
For every \(J\subseteq W\), the set \(D\cup J\) is an ideal. Label \(D\) positively and \(V\setminus(D\cup W)\) negatively. The labels on \(W\) are then independently selectable: every choice is a prefix in some linear extension of \(P\).

At each arrival in \(W\), choose the opposite of the deterministic prediction. Finally choose a linear extension in which the resulting positive ideal comes first, and take its corresponding threshold. \qed

\subsubsection*{Lemma D.5 --- Extreme-point padding}

Under E, suppose one of five points, \(p\), can be fixed as either

\begin{itemize}
\item the least instance, with target label \(1\); or
\item the greatest instance, with target label \(0\),
\end{itemize}

while the other four instances have arbitrary relative order and arbitrary E-threshold target. With at most one further query, at least three mistakes can be forced.

\textbf{Proof.} This embeds the four-point E problem, to which Theorem~\ref{thm:A} gives
\[
M\ge4-1=3.
\]
It remains to check that queries involving the padding point do not invalidate the embedding.

If \(p\) is least:

\begin{itemize}
\item a query requiring \(p=0\) and some ordinary instance to be \(1\) is impossible;
\item if it requires \(p=0\) but no ordinary positive, return the empty-instance prefix;
\item otherwise, make at most one base ERM call and extend its vector by \(p=1\).
\end{itemize}

If \(p\) is greatest:

\begin{itemize}
\item a query requiring \(p=1\) and some ordinary instance to be \(0\) is impossible;
\item if it requires \(p=1\) but no ordinary zero, return the all-one concept;
\item otherwise, make at most one base call and extend its vector by \(p=0\).
\end{itemize}

Conflicting labels and a negative sentinel are handled directly by \(\bot\). Previously returned vectors that are valid for every order in the embedded subclass can be cached and returned without a base call.

These are legal, deterministic, memoryless simulations on the full domains. \qed

The same greatest-point simulation also works when the base class uses convention N; this will be used later.

\subsubsection*{Lemma D.6 --- Predicting before resolving a matching costs three mistakes}

Suppose the retained relations are
\[
a\prec b,\qquad c\prec d,
\]
with a fifth point \(e\) otherwise unrelated. If the learner predicts the first label now, the adversary can force that prediction to be wrong and then force two further mistakes, even if the full order is subsequently disclosed.

The conclusion also holds with zero or one retained comparison, since such a state can be strengthened to the displayed matching.

\textbf{Proof.} For any current point \(v\) and either desired label \(y\), construct three future points \(H\) such that:

\begin{enumerate}
\item all four prefix labelings on \(H\) are possible while \(v=y\);
\item the earliest-arriving point of \(H\) is its median.
\end{enumerate}

The constructions below suffice; exchanging the two pairs handles \(v=c,d\).

\begin{center}\small\begin{tabular}{p{0.23\linewidth}p{0.23\linewidth}p{0.23\linewidth}p{0.23\linewidth}}\hline
Current point & Desired label & Choose \(H\) & Order of the groups \\ \hline
\(a\) & \(0\) & \(\{c,d,e\}\) & \(H<a<b\) \\
\(a\) & \(1\) & \(\{c,d,e\}\) & \(a<H<b\) \\
\(b\) & \(0\) & \(\{c,d,e\}\) & \(a<H<b\) \\
\(b\) & \(1\) & \(\{c,d,e\}\) & \(a<b<H\) \\
\(e\) & \(0\) & \(\{b,c,d\}\) & \(a<H<e\) \\
\(e\) & \(1\) & \(\{b,c,d\}\) & \(a<e<H\) \\
\hline\end{tabular}\end{center}

Each triple has just one required comparison, \(c<d\). Its earliest point can always be made the median. For the triple \(\{c,d,w\}\), use respectively
\[
w<c<d,\qquad c<d<w,\qquad c<w<d
\]
when the earliest point is \(c,d,w\).

Choose \(y\) opposite to the current prediction. On the triple, make the median prediction wrong. Its label leaves two possible thresholds differing on one later extreme; make that prediction wrong too. The fourth future point has a fixed label common to all four triple thresholds.

An ERM oracle reveals the class, not the target. Even full knowledge of this order cannot distinguish these remaining target thresholds before their labels arrive. \qed

\subsubsection*{Lemma D.7 --- The second-query alternatives}

Suppose the only retained comparison is \(a<b\). Any query can be answered so that the remaining orders contain either:

\begin{enumerate}
\item a subclass covered by Lemma D.5; or
\item all extensions of two disjoint comparisons.
\end{enumerate}

\textbf{Proof.} Remove the sentinel from the query after checking its label, and handle conflicts, empty queries, and single-sided queries directly. Such a query can receive a no-information answer, after which fixing \(a\) as the least instance gives case 1.

For a genuine mixed query, let \(A\ne\varnothing\) be its positive set and \(B\ne\varnothing\) its negative set.

If there are \(p\in A,n\in B\) such that adding \(n<p\) is consistent with \(a<b\) and creates no three-element chain, answer \(\bot\) and retain this relation. The resulting poset is one of:

\begin{itemize}
\item one comparison;
\item a two-edge fork;
\item two disjoint comparisons.
\end{itemize}

For a fork, put its common source globally first, or its common sink globally last. The other four points are unrestricted, giving case 1. Disjoint comparisons give case 2. If the added relation is already $a<b$, fix $a$ as the least instance and set its target label to $1$; the other four instances retain arbitrary relative order and arbitrary E-prefix targets, which satisfies the pre-existing comparison and makes the present $\bot$ universally legal, so this is case 1 as well.

Otherwise, \(a\notin B\) and \(b\notin A\): either membership would provide a permitted reversed pair. Moreover, if both \(A\setminus\{a\}\) and \(B\setminus\{b\}\) were nonempty, a pair between these sets would be disjoint from \(a<b\), again permitted. Hence
\[
A=\{a\}\quad\text{or}\quad B=\{b\}.
\]
In the first case return the prefix \(\{a\}\); in the second return \(V\setminus\{b\}\). These are respectively consistent with making \(a\) least or \(b\) greatest, leaving the other four points unrestricted. \qed

\subsubsection*{Lemma D.8 --- One query cannot eliminate a three-point antichain from the matching}

From the matching state
\[
a<b,\qquad c<d,
\]
plus isolated \(e\), every query has a legal answer leaving a poset with a three-point antichain.

\textbf{Proof.} Put
\[
L=\{a,c\},\qquad R=\{b,d\}.
\]
Trivial queries can receive a no-information answer. For a mixed query with positive set \(A\) and negative set \(B\):

\begin{itemize}
\item If \(B\cap L\ne\varnothing\), choose \(\ell\in B\cap L\) and \(p\in A\). Answer \(\bot\), retaining \(\ell<p\). The set \(R\cup\{e\}\) remains an antichain.
\item Otherwise, if \(A\cap R\ne\varnothing\), choose \(r\in A\cap R\) and \(n\in B\). Answer \(\bot\), retaining \(n<r\). The set \(L\cup\{e\}\) remains an antichain.
\item Otherwise,
\end{itemize}
  \[
  A\subseteq L\cup\{e\},\qquad B\subseteq R\cup\{e\}.
  \]
  Return the prefix
  \[
  H=
  \begin{cases}
  L,&e\in B,\\
  L\cup\{e\},&e\notin B.
  \end{cases}
  \]
  This respects the matching and the query. Its three-point side is an antichain.

In each \(\bot\) case, a requested negative has been placed before a requested positive, making the query genuinely unrealizable. \qed

\subsubsection*{Proposition D.9 --- \(M_E^*(5,3)\ge3\)}

\textbf{Proof.} Fix any deterministic learner making at most three calls.

\begin{enumerate}
\item \textbf{Before the first call.} If it predicts, Lemma D.6 forces three mistakes. Otherwise, answer its first mixed query by \(\bot\), retaining one reversed requested pair \(a<b\). A trivial query can receive a no-information answer, and one comparison may then be supplied for free.
\end{enumerate}

\begin{enumerate}
\item \textbf{After the first call.} If it predicts, Lemma D.6 again applies. Otherwise apply Lemma D.7 to its second call. An extreme-point/fork outcome invokes Lemma D.5, with at most one call left. The only remaining case is two disjoint comparisons.
\end{enumerate}

\begin{enumerate}
\item \textbf{In the matching case.} Predicting before the last call loses by Lemma D.6. If the learner makes the last call, Lemma D.8 leaves a three-point antichain. No queries remain, so Lemma D.4 forces three mistakes.
\end{enumerate}

This covers algorithms that use fewer than three calls as well.

\textbf{Oblivious solidification and memorylessness.} Cache every queried set and its answer; repeated queries receive their cached answer. Such repetitions consume budget without imposing additional restrictions.

Every returned nontrivial vector was retained as a prefix cut. Every nontrivial \(\bot\) answer has a retained reversed pair witnessing impossibility. At the end, choose a total-order extension and a target prefix realizing the constructed labels. Define
\[
O^*(S)=
\begin{cases}
\text{the cached answer},&S\text{ was queried},\\
\text{the least consistent threshold in the final order},&S\text{ is realizable and unqueried},\\
\bot,&\text{otherwise}.
\end{cases}
\]
All cached answers are legal in the final order. In the padding branch, use the fixed oracle from Theorem~\ref{thm:A} through the explicit simulation of Lemma D.5.

Thus \(O^*\) is a legal function of \(S\) alone. Determinism makes the replay identical. The resulting order, target, sequence, and oracle are all fixed before that replay, as required by the oblivious-instance requirement.

All choices can be tie-broken by arrival index. Final order extensions can be obtained by lexicographic topological sorting, with the target ideal first. This is a finite constructive solidification, not a compactness argument. \qed

Lemma D.3 supplies the matching upper bound, proving
\[
\boxed{M_E^*(5,3)=3}.
\]

\medskip

\subsection*{4. Embedding the known-class lower bound in the transductive model}

\subsubsection*{Lemma D.10 --- Known-order lower bounds}

For every query budget,
\[
M_E^*(T,Q)\ge \lfloor\log_2(T+1)\rfloor,
\qquad
M_N^*(T,Q)\ge \lfloor\log_2T\rfloor.
\]

\textbf{Proof under E.} Put \(d=\lfloor\log_2(T+1)\rfloor\) and choose \(2^d-1\) ordered instances. Restrict the target to their \(2^d\) prefixes, including the empty prefix. Present these instances in breadth-first order of the perfectly balanced binary search tree; present any extra, always-negative instances afterwards.

The entire sequence is fixed in advance. At each depth, the node on the currently surviving search path splits the possible target cuts equally. Choose its label opposite to the deterministic prediction. Other nodes encountered outside the surviving interval have labels common to all surviving targets. This forces one mistake at each of the \(d\) depths.

The full order may be disclosed for free. Fix, for example, the least-consistent-threshold ERM rule. Given the order, all oracle answers are simulatable and contain no information about which target cut was chosen.

The adaptive choice of target cuts is solidified by selecting the final surviving target and replaying the deterministic interaction.

\textbf{Under N}, reserve the least instance as an always-positive point. Use \(2^d-1\) further instances for the same construction, where \(d=\lfloor\log_2T\rfloor\). The reserved minimum supplies the ``empty'' prefix on the hard instances while the target on the full domain remains nonempty. \qed

This is a worst-sequence statement. It does \textbf{not} assert that every known-order instance permutation has this mistake complexity.

\medskip

\subsection*{5. The N convention also has a strict counterexample}

\subsubsection*{Lemma D.11 --- First-action lower recurrence under N}

For \(n\ge2\) and \(q\ge1\),
\[
M_N^*(n,q)\ge
\min\left\{
M_E^*(n-1,q-1),\
1+M_E^*(n-1,q),\
1+M_N^*(n-1,q)
\right\}. \tag{D.3}
\]

\textbf{Proof.}

If the first action is a query, it can always be answered while leaving some point \(p\) globally least and all other relative orders unrestricted:

\begin{itemize}
\item Mixed positive/negative query: choose \(p\) among its negatives and answer \(\bot\). A globally least point cannot be negative under N.
\item Negative-only query on a proper subset: choose \(p\) outside that subset and return the singleton prefix \(\{p\}\).
\item Negative-only query on the whole domain: answer \(\bot\).
\item Positive-only or empty query: return the all-one concept.
\item Conflicting labels: answer \(\bot\).
\end{itemize}

Disclose \(p\) as the minimum for free. Its target label is necessarily \(1\); the other \(n-1\) points form an E problem with at most \(q-1\) calls remaining. Cached first answers are universally legal in this subclass.

If the first action is prediction \(0\), give label \(1\), make that point least, and retain an E problem on the remaining \(n-1\) points.

If the first action is prediction \(1\), give label \(0\), make that point greatest, and retain an N problem on the remaining \(n-1\) points. The additional all-one concept introduced by the greatest point is handled by the simulation in Lemma D.5.

These three cases prove (D.3). \qed

Apply (D.3) with \(n=6,q=4\). Proposition D.9 and Lemma D.10 give
\[
\begin{aligned}
M_E^*(5,3)&=3,\\
1+M_E^*(5,4)&\ge1+\lfloor\log_2 6\rfloor=3,\\
1+M_N^*(5,4)&\ge1+\lfloor\log_2 5\rfloor=3.
\end{aligned}
\]
Therefore \(M_N^*(6,4)\ge3\). Lemma D.3, with internal cap \(k=3\), gives the reverse inequality:
\[
\boxed{M_N^*(6,4)=3}.
\]

Consequently, any proposed formula assigning the value \(2\) to \((T,Q)=(6,4)\) is false.

\subsubsection*{Further N conclusions}

Equation (D.3), Theorem~\ref{thm:A} under E, and induction on \(T\) yield the strengthened \textbf{budget} lower bound (this is a budget statement for \(Q\ge1\), obtained by the first-action recurrence and induction rather than directly from the pathwise bound of Theorem~\ref{thm:A} under N: for \(Q\ge1\) the three branches of (D.3) give \(M_E^*(T-1,Q-1)\ge T-Q\) and \(1+M_E^*(T-1,Q)\ge T-Q\) by Theorem~\ref{thm:A} under E, and \(1+M_N^*(T-1,Q)\ge T-Q\) by the induction hypothesis, while the base case \(T=1\) has a nonpositive right-hand side)
\[
M_N^*(T,Q)\ge T-Q\qquad(Q\ge1).
\]
Together with Lemma D.10,
\[
M_N^*(T,Q)\ge
\begin{cases}
T-1,&Q=0,\\
\max\{T-Q,\lfloor\log_2T\rfloor\},&Q\ge1.
\end{cases} \tag{D.4}
\]

Predicting \(1\) throughout makes at most \(T-1\) mistakes under N. Consequently,
\[
\boxed{M_N^*(T,0)=M_N^*(T,1)=T-1.} \tag{D.5}
\]
This is a budget statement; it is \textbf{not} a new pathwise assertion \(M+q\ge T\) when the actual query count might be zero.

\medskip

\subsection*{6. Valid bounds after rejecting the claimed exact curve}

The results established above imply
\[
\max\{T-Q,\lfloor\log_2(T+1)\rfloor\}
\le M_E^*(T,Q)
\le \max\{T-Q,\lceil T/2\rceil\}, \tag{D.6}
\]
and the N lower bound (D.4), together with
\[
M_N^*(T,Q)
\le
\min\left\{T-1,\ \max\{T-Q,\lceil T/2\rceil\}\right\}. \tag{D.7}
\]

There is also exact large-budget saturation:
\[
Q\ge2(T-1)
\quad\Longrightarrow\quad
M_E^*(T,Q)=\lfloor\log_2(T+1)\rfloor,\quad
M_N^*(T,Q)=\lfloor\log_2T\rfloor.
\]
For completeness, the explicit count follows from the sorting method of AHR25 (Thm.~4.5(3)): query a pair in one orientation, and if necessary in the reverse orientation. A successful return splits the current set into two nonempty ordered parts. There are \(T-1\) internal splits, each costing at most two calls. Once sorted, maintain the set \(V_t\) of prefix concepts consistent with the labels seen so far and predict its majority label; every mistake at least halves \(|V_t|\), while the true target is never removed, so \(1\le|V_0|2^{-M}\) and \(M\le\lfloor\log_2|V_0|\rfloor\), with \(|V_0|=T+1\) under E and \(|V_0|=T\) under N. (The sufficient budget \(2(T-1)\) is not claimed to be minimal; for instance \(M_E^*(2,1)=1\): one pair query fixes the order of two points and halving then makes one mistake, matching the lower bound.)

Equations (D.6)--(D.7) are bounds, not an exact characterization. The strict finite counterexamples show why the conjectured equality cannot be assembled from Theorem~\ref{thm:A} and the known-class Halving bound alone.

Theorem~4.4 of AHR25 is not invoked here; its statement, and why its hypothesis fails for the threshold family, are discussed once in Appendix~\ref{sec:proofA}. (For the E counterexample, \(T=5\) and \(d_{\rm LD}=2\), so its hypothesis \(T\ge 2^{d_{\rm LD}+1}=8\) fails.) Theorem~\ref{thm:A} is invoked as a proved theorem; the deterministic/randomized separation is unaffected.

\medskip

\section{Proof of Theorem~\ref{thm:D}}
\label{sec:proofD}
\subsection*{Theorem~\ref{thm:D}: deterministic linear-query learning with a weak consistency oracle}

\textbf{Interface and conventions.} Throughout this section, the learner accesses the class only through the Boolean oracle of \textbf{Definition 2.2 of AHR25}. It never receives or evaluates a returned concept. The concept-returning interface of Section 2 is used only in the later comparison table.

Let the \(T\ge1\) instances be distinct and revealed in advance, as stipulated in the transductive protocol. Labels are generated by an arbitrary target
\[
c_z(x)=\mathbf 1[x\preceq z]
\]
under an arbitrary unknown total order \(\preceq\).

Write \(\epsilon=0\) under convention E and \(\epsilon=1\) under convention N:

\begin{itemize}
\item \textbf{E:} the domain contains a point preceding all \(T\) instances, so the empty prefix on the instances is available.
\item \textbf{N:} the domain consists exactly of the \(T\) instances, so the empty prefix is unavailable.
\end{itemize}

\subsubsection*{Theorem~\ref{thm:D} --- explicit bounds}

There are deterministic algorithms using only weak consistency queries such that, for every admissible domain, total order, target, and transductive sequence,
\[
\begin{array}{c|c|c}
\text{Convention}&\text{Number of queries}&\text{Number of mistakes}\\ \hline
E&Q\le 66T&M\le\lfloor\log_2(T+1)\rfloor\\[1mm]
N&Q\le 67(T-1)&M\le\lfloor\log_2T\rfloor.
\end{array}
\]

Every oracle query made by these algorithms has exactly the form
\[
\{(u,1),(v,0)\},\qquad u\ne v.
\]

Consequently, in either convention,
\[
Q\le67T,\qquad M\le\log_2T+1.
\]

These are pathwise bounds: there is no randomization and no averaging over oracle answers.

\medskip

\subsection*{Lemma E.1 --- comparison interface and adaptive consistency}

For distinct \(u,v\in X\),
\[
\operatorname{WC}\bigl(\{(u,1),(v,0)\}\bigr)=\text{realizable}
\quad\Longleftrightarrow\quad u\prec v.
\]

This holds under both E and N.

\textbf{Proof.} Realizability means that some \(z\in X\) satisfies
\[
u\preceq z\prec v,
\]
which implies \(u\prec v\). Conversely, if \(u\prec v\), the concept \(c_u\) realizes the query. No empty-prefix assumption is needed. \qed

The query concerns the existence of \textbf{some} consistent concept, not consistency with the actual target. In particular, observed target labels must \textbf{not} be appended to this comparison query.

The comparison implementation is also valid against adaptively supplied comparison answers, provided the answers remain consistent with a total order. Indeed, fix any completed total order consistent with a finite comparison transcript. A deterministic comparison algorithm, run against that fixed order, follows exactly the same transcript. Thus every correctness and comparison-count guarantee proved for arbitrary fixed total orders holds for every consistent adaptive transcript.

For a finite instance set, a consistent partial order has a total-order extension: repeatedly remove a minimal remaining element. This is the only extension fact needed here.

\medskip

\subsection*{Lemma E.2 --- exact selection in at most \(32n\) comparisons}

There is a deterministic comparison algorithm
\[
\operatorname{SELECT}(A,k),\qquad 1\le k\le n=|A|,
\]
which returns the element of rank \(k\) using at most \(32n\) comparisons.

Therefore, an exact median of rank
\[
k=\left\lceil\frac n2\right\rceil
\]
can be found using at most \(32n\) weak consistency queries, and partitioning the other elements relative to it costs at most \(n-1\) additional queries.

\textbf{Algorithm.}

\begin{itemize}
\item If \(n\le64\), insertion-sort the elements and return the element of rank \(k\).
\item Otherwise, let \(m=\lfloor n/5\rfloor\). Form \(m\) complete groups of five; leave at most four elements outside these groups.
\item Insertion-sort each complete group and collect its median.
\item Recursively select the median \(q\) of those \(m\) medians, using rank \(a=\lceil m/2\rceil\).
\item Compare every other original element with \(q\).
\item Return \(q\), or recurse into the appropriate side with the appropriately adjusted rank.
\end{itemize}

All grouping and tie-independent choices use the original arrival-index order, making the procedure deterministic.

\textbf{Correctness.} The partition determines the exact rank of \(q\). The desired element is therefore either \(q\) itself or the desired adjusted rank in exactly one of the two sides. Induction proves correctness.

\textbf{Comparison count.} Let \(C(n)\) be the worst-case number of comparisons, with \(C(0)=0\). For \(n\le64\), insertion sort uses at most
\[
\frac{n(n-1)}2\le\frac{63}{2}n\le32n.
\]

Now suppose \(n\ge65\). Sorting the complete groups costs at most \(10m\) comparisons, and partitioning around \(q\) costs at most \(n-1\).

There are at least
\[
3(a-1)+2=3a-1
\]
elements strictly below \(q\). There are at least
\[
3(m-a)+2\ge3a-1
\]
elements strictly above \(q\); the last inequality follows directly for both parities of \(m\).

Thus the side used by the second recursive call has size \(s\) satisfying
\[
s\le n-3a.
\]
Consequently,
\[
m+s
\le n+m-3\left\lceil\frac m2\right\rceil
\le n-\frac m2
\le \frac{9n}{10}+\frac25,
\]
where the last step uses \(m\ge(n-4)/5\).

Induction now gives
\[
\begin{aligned}
C(n)
&\le C(m)+C(s)+10m+n-1\\
&\le32(m+s)+3n-1\\
&\le32\left(\frac{9n}{10}+\frac25\right)+3n-1\\
&=\frac{159n+59}{5}\\
&\le32n,
\end{aligned}
\]
because \(n\ge65>59\). This proves the stated explicit constant. \qed

For \(n\ge2\), the selected rank satisfies
\[
\frac{3n}{10}\le\left\lceil\frac n2\right\rceil\le\frac{7n}{10}.
\]
For \(n=2\) this is immediate; for \(n\ge3\), use
\(\lceil n/2\rceil\le(n+1)/2\le7n/10\).
The literal middle-rank interval contains no integer when \(n=1\); a singleton instead requires no comparisons.

\medskip

\subsection*{Lemma E.3 --- a phase contains at most one mistake}

The learner maintains a partial map \(K\) of \textbf{certified labels}. Every certification will be proved correct. Let
\[
U=\{x_1,\ldots,x_T\}\setminus\operatorname{dom}(K).
\]
Thus every point in \(U\) is unobserved and uncertified. Certification may be conservative: \(U\) need not contain exactly the logically undetermined points, and it need not be an interval in the unknown order.

At the start of a phase, freeze
\[
V=U,\qquad n=|V|\ge1.
\]
Using Lemma E.2, find the element \(p\) of rank \(k=\lceil n/2\rceil\) in \(V\), and partition
\[
L=\{u\in V:u\prec p\},\qquad
R=\{u\in V:p\prec u\}.
\]

Throughout this phase:

\begin{itemize}
\item Predict the certified label on points already in \(\operatorname{dom}(K)\).
\item On an uncertified point in \(L\cup\{p\}\), predict \(1\).
\item On an uncertified point in \(R\), predict \(0\).
\item Certify each observed label.
\item Following a mistake:
\item if the prediction was \(1\), certify every point of \(\{p\}\cup R\) as \(0\);
\item if the prediction was \(0\), certify every point of \(L\cup\{p\}\) as \(1\).
\end{itemize}

End the phase, before another prediction, whenever
\[
|U|\le \rho(n),
\qquad
\rho(n):=\left\lfloor\frac{n-1}{2}\right\rfloor.
\]

Then every certification is correct, and each phase contains at most one mistake.

\textbf{Proof.} Certified points cannot cause mistakes, assuming the certification invariant.

Consider a mistake on an uncertified point \(x\). The following table lists all possibilities.

\begin{center}\small\begin{tabular}{p{0.30\linewidth}p{0.30\linewidth}p{0.30\linewidth}}\hline
Mistake & Valid threshold implication & Active set after certification \\ \hline
\(x\in L\), prediction \(1\), label \(0\) & Every point in \(\{p\}\cup R\) is \(0\) & \(U'\subseteq L\setminus\{x\}\) \\
\(x=p\), prediction \(1\), label \(0\) & Every point in \(\{p\}\cup R\) is \(0\) & \(U'\subseteq L\) \\
\(x\in R\), prediction \(0\), label \(1\) & Every point in \(L\cup\{p\}\) is \(1\) & \(U'\subseteq R\setminus\{x\}\) \\
\hline\end{tabular}\end{center}

For example, in the first case,
\[
z\prec x\prec p,
\]
so \(p\) and every point above it are negative. In the third case,
\[
p\prec x\preceq z,
\]
so \(p\) and every point below it are positive. These implications also show that new certifications never conflict with previous correct certifications.

Since \(|L|=k-1\) and \(|R|=n-k\), the corresponding size bounds are
\[
k-2,\qquad k-1,\qquad n-k-1.
\]
Cases involving an empty side simply cannot occur. In every possible case,
\[
|U'|
\le\max\{k-1,n-k-1\}
=\left\lfloor\frac{n-1}{2}\right\rfloor
=\rho(n).
\]

Therefore, the first mistake immediately ends the phase. There can be no second mistake in that phase.

Importantly, these are containments in the \textbf{original phase sets} \(L,R\). They remain valid even if many points have already been removed by correct predictions. No lower bound on the number of \emph{newly} removed points is required. \qed

\medskip

\subsection*{Proof of Theorem~\ref{thm:D}}

\subsubsection*{Convention E}

Initially, \(K\) is empty and the active set has size \(T\). Repeatedly execute the phases of Lemma E.3. If \(U\) becomes empty, all remaining predictions use certified labels.

Let the positive phase-start sizes be
\[
n_1,n_2,\ldots,n_J.
\]
Then \(n_1=T\), and the phase-ending rule gives
\[
n_{i+1}\le\left\lfloor\frac{n_i-1}{2}\right\rfloor,
\qquad\text{hence}\qquad
n_{i+1}+1\le\frac{n_i+1}{2}.
\]

Since \(n_J\ge1\),
\[
2\le n_J+1\le\frac{T+1}{2^{J-1}},
\]
which implies
\[
J\le\lfloor\log_2(T+1)\rfloor.
\]
Lemma E.3 therefore yields
\[
M\le J\le\lfloor\log_2(T+1)\rfloor.
\]

The query cost of a phase of size \(n_i\) is at most
\[
32n_i+(n_i-1)\le33n_i.
\]
Furthermore, \(n_{i+1}\le n_i/2\), so
\[
\sum_{i=1}^{J}n_i\le2T.
\]
Thus
\[
Q\le33\sum_i n_i\le66T.
\]

\subsubsection*{Convention N}

Here \(X\) consists exactly of the \(T\) instances. First find its minimum \(a\) by a sequential minimum scan, using exactly \(T-1\) comparisons.

Every target satisfies \(a\preceq z\), so \(c_z(a)=1\). Certify \(a\) as positive before online prediction begins.

Run the same phase algorithm on the remaining \(T-1\) active points. The preceding proof, with initial active size \(T-1\), gives
\[
M\le\lfloor\log_2T\rfloor
\]
and
\[
Q\le(T-1)+66(T-1)=67(T-1).
\]

For \(T=1\), this means zero queries and zero mistakes under N. Under E, the singleton phase uses zero queries and makes at most one mistake.

Consequently, under either convention, \(Q\le 67T\) and \(M\le\lfloor\log_2(T+1)\rfloor\le\log_2T+1\); the sharper convention-specific bounds are the ones stated above. \qed

\textbf{Convention correction.} The E algorithm itself also works unchanged under N, with the same \(66T\) query bound and at most \(\lfloor\log_2(T+1)\rfloor\) mistakes. The minimum-scan variant obtains the sharper N bound \(\lfloor\log_2T\rfloor\). Thus the integer mistake correction is exactly the replacement of \(T+1\) by \(T\), a difference of at most one, matching the two Littlestone dimensions of Section 2.

\medskip

\subsection*{Weak consistency: matching linear lower bounds}

\subsubsection*{The quoted lower bound}

The statement of Theorem 4.3 of AHR25 is:

\begin{quote}Consider any family \(\mathcal F\) of concept classes of the form \(C\subset\{0,1\}^X\), where the family \(\mathcal F\) has the property that for every labeling function \(f:\{x_1,x_2,\ldots,x_T\}\to\{0,1\}\), there exists some concept class \(C\in\mathcal F\) and some concept \(c\in C\) such that \(c(x_t)=f(x_t)\) for all \(t\in[T]\) (i.e., all \(2^T\) possible binary labelings of the sequence \(x\) are captured by the family \(\mathcal F\)). For \(T\ge 100\), any (possibly randomized) algorithm that makes at most \(T/20\) queries to the weak consistency oracle will incur an expected mistake bound of at least \(T/20\). In particular, this theorem holds for general classes like Littlestone classes, and also for special families of classes, including thresholds, \(k\)-intervals, and \(d\)-Hamming balls (see the next section).\end{quote}

\subsubsection*{Applicability under E}

Use the finite domain
\[
X=\{s,x_1,\ldots,x_T\},
\]
with \(s\) preceding all instances. For an arbitrary labeling \(f\), order the points as
\[
s
\prec \{\text{positive instances, in arrival-index order}\}
\prec \{\text{negative instances, in arrival-index order}\}.
\]
If there is a positive instance, choose the last positive instance as \(z\); otherwise choose \(z=s\).

This realizes every one of the \(2^T\) labelings. Therefore, for \(T\ge100\), a hard cap of \(T/20\) weak consistency queries entails a worst-case expected mistake bound of at least \(T/20\). Deterministic algorithms are included as a special case.

\subsubsection*{Applicability under N: the necessary one-point correction}

The full \(T\)-point N family does \textbf{not} realize the all-zero labeling, so the hypothesis of Theorem 4.3 of AHR25 should not be applied to those \(T\) points without adjustment.

Instead, designate the first actual instance \(s\) as the minimum and restrict to orders with that property. Its label is always \(1\). On the remaining \(T-1\) instances, every binary labeling is realizable by arranging positives before negatives and choosing \(z=s\) when all remaining labels are zero.

Given a learner for the full N problem, simulate it on
\[
(s,x_2,\ldots,x_T),
\]
supply the known first label \(1\), and count its mistakes on the remaining sequence. Queries are unchanged, and the remaining mistakes are no greater than its total mistakes.

Theorem 4.3 of AHR25 therefore applies with effective horizon
\[
h=T-1.
\]
For \(T\ge101\), a hard cap of \((T-1)/20\) queries entails at least \((T-1)/20\) expected mistakes in the worst case.

The designated minimum is an actual instance of the full N problem, not an additional unavailable query point.

\subsubsection*{The oracle is fixed and memoryless}

There is no adversarial choice of a Boolean answer once the order is fixed. For completeness, on either finite lower-bound domain, let \(r(x)\in\{1,\ldots,d\}\) be the rank of \(x\). Define
\[
a(S)=\max\bigl(\{1\}\cup\{r(x):(x,1)\in S\}\bigr),
\]
\[
b(S)=\min\bigl(\{d\}\cup\{r(x)-1:(x,0)\in S\}\bigr).
\]
The weak oracle is the deterministic function
\[
\boxed{\operatorname{WC}(S,\preceq)=\mathbf 1[a(S)\le b(S)].}
\]
This handles arbitrary query sizes, empty queries, one-label queries, contradictory labels, and every available domain point.

The lower-bound invocation is Theorem 4.3 of AHR25 in its oblivious-adversary model. No adaptive oracle construction is being substituted for an oblivious instance.

\medskip

\subsection*{Lemma E.4 --- expected-query budgets also require linear queries}
\label{lem:D4}

Let
\[
h=T-\epsilon\ge100.
\]
Suppose a possibly randomized learner has worst-instance expected mistake bound \(m\) and worst-instance expected query bound \(q\). Then
\[
\boxed{m+20q\ge\frac h{20}.}
\]
In particular,
\[
q\ge\frac h{400}-\frac m{20}.
\]

\textbf{Proof.} Work on the effective \(h\)-round problem above, including the N simulation when necessary.

Set
\[
b=\left\lfloor\frac h{20}\right\rfloor.
\]
Construct a truncated learner that follows the original learner, using the same private random bits, until it would make query \(b+1\). At that point it makes no further queries and uses an arbitrary fixed prediction for all remaining rounds.

The truncated learner has a hard query cap \(b\). On any fixed instance, let \(A\) be the event that truncation occurs. Coupling the two runs gives
\[
M_{\mathrm{trunc}}\le M_{\mathrm{original}}+h\mathbf 1_A.
\]
Also,
\[
\Pr(A)\le\frac{\mathbb E[Q_{\mathrm{original}}]}{b+1},
\]
because \(A\) implies \(Q_{\mathrm{original}}\ge b+1\).

Theorem 4.3 of AHR25 supplies a fixed instance on which
\[
\mathbb E[M_{\mathrm{trunc}}]\ge\frac h{20}.
\]
On that same instance,
\[
\frac h{20}
\le m+\frac h{b+1}q
\le m+20q.
\]
This proves the claim. \qed

Thus a logarithmic expected mistake guarantee requires
\[
q\ge\frac{T-\epsilon}{400}-O(\log T)=\Omega(T).
\]
For example, the hard-cap statement alone implies that
\[
q\le\frac h{800}
\quad\Longrightarrow\quad
m\ge\frac h{40}
\]
under an \textbf{expected-query} budget.

If the unqualified query budget of Theorem 4.3 of AHR25 is read directly as the expected quantity \(Q\) defined in Section 2, its quoted \(T/20\) implication applies directly to that budget as well. The argument above avoids needing that interpretation: the required expected-query \(\Omega(T)\) lower bound follows even from the hard-cap reading alone.

\subsubsection*{Corollary E.5 --- tight weak-consistency query order}

For either E or N, the query complexity needed to guarantee \(O(\log T)\) mistakes is
\[
\boxed{\Theta(T)}
\]
for both deterministic and randomized learners.

For randomized learners this remains true when query and mistake budgets are worst-instance expectations. The upper bound may be taken from Theorem 4.5(2) of AHR25, or, more strongly regarding query tails, from the deterministic algorithm of Theorem~\ref{thm:D}.

\medskip

\subsection*{Interface comparison and the precise separation statement}

The table concerns \(O(\log T)\) mistakes. Randomized query bounds use worst-instance expectation, as in the cost accounting.

\begin{center}\small\begin{tabular}{p{0.30\linewidth}p{0.30\linewidth}p{0.30\linewidth}}\hline
Oracle interface & Deterministic query complexity & Randomized query complexity \\ \hline
Consistency-type ERM: returns a full concept or \(\bot\) & \textbf{\(\Theta(T)\)} for worst-case legal oracles and for the extremal rules of Theorem~\ref{thm:Aprime} (not uniform over all pre-declared rules; see Theorem~\ref{thm:rule}) --- Theorems~\ref{thm:Aprime} and~\ref{thm:A} give \(Q\ge T-\epsilon-O(\log T)\); the sorting algorithm of AHR25 gives \(O(T)\) & \textbf{\(\Theta(\log T)\)} --- the \(O(\log T)\) upper bound is AHR25, Thm.~4.5(4); the matching lower bound is Theorem~\ref{thm:C} \\
Weak consistency: returns only realizable / not realizable & \textbf{\(\Theta(T)\)} --- Theorem~\ref{thm:D} together with Theorem 4.3 of AHR25 & \textbf{\(\Theta(T)\)} --- Theorem 4.5(2) of AHR25 (or Theorem~\ref{thm:D}) together with Theorem 4.3 of AHR25 and Lemma E.4 \\
\hline\end{tabular}\end{center}

\subsubsection*{Interpretive proposition --- scope of the interface separation}

In the model of Section~2 and for logarithmic mistake guarantees on unknown-order thresholds:

\begin{enumerate}
\item Deterministic learners require linear-order query complexity under either interface: at the weak consistency interface unconditionally, and at the ERM interface against worst-case legal oracles and against the pre-declared extremal rules of Theorem~\ref{thm:Aprime}. The ERM statement is not uniform over all legal pre-declared selection rules: the feasible-median rule admits deterministic \(O(\log T)\) queries and mistakes (Theorem~\ref{thm:rule}).
\item Randomized learners can use \(O(\log T)\) expected queries when a successful query returns an evaluable full concept, by AHR25, Thm.~4.5(4).
\item Randomized learners still require \(\Omega(T)\) expected queries when only the realizability bit is returned.
\end{enumerate}

Thus the established linear-versus-logarithmic improvement from randomization is available with the \textbf{concept-returning ERM interface}, but disappears at the level of query order with the \textbf{weak consistency interface}.

This conclusion does not require Theorem~\ref{thm:C}. Theorem~\ref{thm:C} is needed only to upgrade the randomized ERM upper bound to a matching \(\Theta(\log T)\) characterization. The separation is exponential when expressed in the parameter \(d=\Theta(\log T)\).

The statement is about these two interfaces, this family, and this mistake regime. It does not assert that randomization cannot improve constants or other weak-oracle tradeoffs.

Theorem~4.4 of AHR25 is not invoked here; its statement, and why its hypothesis fails for the threshold family, are discussed once in Appendix~\ref{sec:proofA}. The weak-oracle lower bound above comes from Theorem 4.3 of AHR25, and the deterministic ERM lower bound from Theorems~\ref{thm:Aprime} and~\ref{thm:A}.

\medskip

\end{document}